\documentclass[lettersize,journal]{IEEEtran}
\usepackage{amsmath,amsfonts} 
\usepackage{array}

\usepackage{textcomp}
\usepackage{stfloats}
\usepackage{url}
\usepackage{verbatim}
\usepackage{graphicx}
\usepackage{cite}
\usepackage{pifont}
\usepackage{booktabs} 
\usepackage{bm}
\usepackage[T1]{fontenc} 
\usepackage[linesnumbered,ruled,vlined]{algorithm2e}
\usepackage{hyperref}
\usepackage{subcaption}
\usepackage{orcidlink}
\usepackage{makecell}

\newcommand{\Var}{\operatorname{Var}}
\newcommand{\bE}{\mathbb{E}}

\allowdisplaybreaks

\usepackage{algorithmic}
\usepackage{amsthm} 
\newtheorem{theorem}{Theorem} 
\newtheorem{lemma}{Lemma}

\begin{document}

\captionsetup[figure]{labelformat=simple, labelsep=period} 

\captionsetup[table]{name={TABLE},labelsep=space} 

\title{Double Actor-Critic with TD Error-Driven Regularization in Reinforcement Learning}

\author{Haohui Chen~\orcidlink{0000-0001-9660-0948}, Zhiyong Chen~\orcidlink{0000-0002-2033-4249}, Aoxiang Liu~\orcidlink{0009-0000-8047-1192}, and Wentuo Fang~\orcidlink{0000-0002-0374-3534} 

\thanks{Haohui Chen, Aoxiang Liu, and Wentuo Fang are with the School of Automation, Central South University, Changsha 410083, China (e-mail: \href{mailto:haohuichen@csu.edu.cn}{haohuichen@csu.edu.cn}; \href{mailto:ordisliu@csu.edu.cn}{ordisliu@csu.edu.cn};\href{mailto:wentuo.fang@outlook.com}{wentuo.fang@outlook.com}).}
\thanks{Zhiyong Chen is with the School of Engineering, University of Newcastle, Callaghan, NSW 2308, Australia (e-mail: \href{mailto:zhiyong.chen@newcastle.edu.au}{zhiyong.chen@newcastle.edu.au}).}
\thanks{Corresponding author: Zhiyong Chen.}}

\maketitle

\begin{abstract}
To obtain better value estimation in reinforcement learning, we propose a novel algorithm based on the double actor-critic framework with temporal difference error-driven regularization, abbreviated as TDDR. TDDR employs double actors, with each actor paired with a critic, thereby fully leveraging the advantages of double critics. Additionally, TDDR introduces an innovative critic regularization architecture. Compared to classical deterministic policy gradient-based algorithms that lack a double actor-critic structure, TDDR provides superior estimation. Moreover, unlike existing algorithms with double actor-critic frameworks, TDDR does not introduce any additional hyperparameters, significantly simplifying the design and implementation process. Experiments demonstrate that TDDR exhibits strong competitiveness compared to benchmark algorithms in challenging continuous control tasks.
\end{abstract}

\begin{IEEEkeywords}
Reinforcement learning, Actor-critic, Double actors, Critic regularization, Temporal difference
\end{IEEEkeywords}

\section{Introduction}

\IEEEPARstart{R}{einforcement} learning (RL) has gained significant attention as a versatile research area, with applications spanning recommender systems \cite{yunDoublyConstrainedOffline2024}, autonomous driving \cite{tothTabularQlearningBased2023}, controller design \cite{pengModelBasedChanceConstrainedReinforcement2024}, and medical fields \cite{collomb-clercHumanThalamicLowfrequency2023}. One of the most well-known RL algorithms, deep Q-learning (DQN), has demonstrated success in learning optimal policies \cite{mnihHumanlevelControlDeep2015}. However, DQN struggles in continuous control scenarios, particularly in environments with high-dimensional action spaces \cite{anschelAveragedDQNVarianceReduction}, \cite{zhangLearningAutomataBasedMultiagent2021}.

In contrast, actor-critic (AC) algorithms, such as trust region policy optimization (TRPO) \cite{schulmanTrustRegionPolicy2017}, proximal policy optimization (PPO) \cite{schulmanProximalPolicyOptimization2017}, and deep deterministic policy gradient (DDPG) \cite{lillicrapContinuousControlDeep2019}, are better suited for continuous control tasks. Further insights into AC-based algorithms are discussed in other works \cite{hanDiversityActorCriticSampleAware}, \cite{banerjeeImprovedSoftActorCritic2024}, \cite{beikmohammadiAcceleratingActorcriticbasedAlgorithms2024}.
However, DDPG and related algorithms are prone to overestimation bias. To address this, Fujimoto et al. introduced the twin delayed deep deterministic policy gradient (TD3) algorithm \cite{fujimotoAddressingFunctionApproximation}, which leverages double critics to reduce overestimation. Notably, TD3 still employs a single critic for updates and action selection.

Building on the double actor-critic (DAC) framework, several advanced algorithms, such as double actors regularized critics (DARC) \cite{lyuEfficientContinuousControl2022}, softmax deep double deterministic policy gradients (SD3) \cite{panSoftmaxDeepDouble}, and generalized-activated deep double deterministic policy gradients (GD3) \cite{lyuValueActivationBias2023}, have introduced different regularization techniques to improve value estimation.

To fully leverage the potential of double critics, we propose the temporal difference error-driven regularization (TDDR) algorithm, built upon the DAC framework. Specifically, TDDR employs clipped double Q-learning (CDQ) with double actors (DA-CDQ), generating four Q-values and utilizing the TD error from their target networks to guide the selection of the appropriate Q-value for critic updates. A more detailed comparison of TDDR with existing algorithms will be provided in Section~\ref{sec:Related Work}.

The primary contributions of TDDR are as follows.

1. TDDR demonstrates strong competitiveness in challenging continuous control tasks, outperforming benchmark algorithms without introducing additional hyperparameters.

2. A key innovation is the introduction of TD error-driven regularization within the DAC framework. Furthermore, TDDR incorporates the CDQ approach, which utilizes double target actors for action evaluation, resulting in CDQ based on double actors.

3. A convergence proof for TDDR is provided under both random and simultaneous updates, and its performance is validated through numerical benchmark comparisons.

The remainder of the paper is organized as follows. Section~\ref{sec:Related Work} reviews the related work and compares the benchmark algorithms with the proposed TDDR. Section~\ref{sec:Preliminaries} covers the preliminaries and motivations for our study. In Section~\ref{sec:Proposed Method}, we provide a detailed explanation of the TDDR algorithm. Section~\ref{sec:Convergence Analysis} presents the convergence analysis of the TDDR algorithm. Section~\ref{sec:experiments} includes experimental results on MuJoCo and Box2d continuous control tasks, along with a discussion of the performance relative to benchmark algorithms. Finally, Section~\ref{sec:Conclusion} concludes the paper and outlines potential future extensions of this work.

\section{Related Work}
\label{sec:Related Work}

This section is divided into two parts: the first part reviews related work, and the second part compares the benchmark algorithms with TDDR.

\subsection{Related Work} 

DDPG and TD3 serve as benchmark algorithms in this paper, both utilizing TD error for updating critic parameters and are widely adopted within the AC framework. Several advanced algorithms build upon this foundational structure.
For example, Wu et al. \cite{wuReducingEstimationBias2020} introduced the triplet-average deep deterministic (TADD) policy gradient algorithm, which incorporates weighted TD error regularization for critic updates. However, TADD introduces two additional hyperparameters compared to TD3, which are environment-dependent and can increase computational complexity, potentially impairing performance if not selected appropriately.

Cheng et al. \cite{chengRobustActorCriticRelative2023} proposed the robust actor-critic (RAC) algorithm, which incorporates constraints on the entropy between policies. While RAC demonstrates some hyperparameter insensitivity in certain environments, its overall performance remains affected by the additional hyperparameters. Similarly, Li et al. \cite{liAlleviatingEstimationBias2022} introduced the co-regularization-based deep deterministic (CoD2) policy gradient algorithm, which alternates the TD targets between DDPG and TD3; however, its performance is highly sensitive to the regularization coefficient.

Conversely, other policy-based regularization algorithms exhibit lower sensitivity to hyperparameter settings \cite{zhouPromotingStochasticityExpressive}. Li et al. \cite{liImprovingExplorationActor2024} introduced the weakly pessimistic value estimation and optimistic policy optimization (WPVOP) algorithm, which employs a pessimistic estimation to adjust the TD target. However, its dependence on hyperparameters varies across environments, leading to increased computational complexity. Cetin et al. \cite{cetinLearningPessimismReinforcement2023} proposed the generalized pessimism learning (GPL) algorithm, which integrates dual TD targets to counteract bias from pessimistic estimation, resulting in improved value estimation with reduced computational overhead.

Additionally, Li et al. \cite{liMultiactorMechanismActorcritic2023} developed the multi-actor mechanism (MAM), which offers multiple action choices within a single state to optimize policy selection. Despite its higher computational cost, MAM effectively enhances value estimation, demonstrating the advantages of multi-actor approaches \cite{lyuEfficientContinuousControl2022}.

While these algorithms have shown promising results, limited research has investigated how DAC can achieve improved value estimation without introducing additional hyperparameters, a key issue addressed in this paper. Notable algorithms in this context, which serve as our benchmark methods built upon the DAC framework, include DARC \cite{lyuEfficientContinuousControl2022}, SD3 \cite{panSoftmaxDeepDouble}, and GD3 \cite{lyuValueActivationBias2023}.

\subsection{Comparison}

The major differences between the benchmark algorithms and the proposed TDDR are summarized in Table~\ref{tab:components}, highlighting their features in key components. A discussion of additional differences follows below, with technical details provided in Sections~\ref{sec:Preliminaries} and \ref{sec:Proposed Method}.
 
The first distinction is the structure of TDDR compared to the benchmark algorithms. TDDR clearly differs from DDPG and TD3 in terms of the number of actors and critics. While the other DAC benchmark algorithms directly apply TD error regularization for critic updates, similar to DDPG and TD3, TDDR introduces a novel form of critic regularization driven by the TD error of its target networks, representing a significant innovation.

Another critical difference lies in the hyperparameters used in the benchmark algorithms compared to TDDR. Specifically, TDDR does not introduce additional hyperparameters beyond those in TD3, while other DAC-based benchmark algorithms do. The introduction of extra hyperparameters can significantly destabilize the learning process if not carefully tuned, often leading to increased standard deviation in performance. Consequently, finding effective hyperparameters presents a challenge in reinforcement learning. Moreover, reducing the number of additional hyperparameters while ensuring improved value estimation is a vital issue that needs to be addressed.

DARC introduces two regularization coefficients, $\lambda$ and $\nu$. The coefficient $\lambda$ constrains the gap between Q-values during critic updates, while $\nu$ is used to mitigate the risk of overestimation. The effectiveness of $\lambda$ relies on achieving significantly improved value estimation. DARC simplifies computation by deriving a single value from the two Q-values produced by the double actors. However, if one Q-value is overestimated, the corrective capability of the other Q-value depends on $\nu$. Both $\lambda$ and $\nu$ can vary across different tasks.

SD3 adds two additional hyperparameters: the number of noise samples (NNS) and the parameter $\beta$. NNS controls the influence of value estimation but shows limited sensitivity in SD3. The parameter $\beta$ is crucial for balancing variance and bias in value estimation. GD3 extends this by introducing four hyperparameters, including NNS, $\beta$, a bias term $b$, and activation functions. While NNS and $\beta$ function similarly in both SD3 and GD3, the bias term $b$ and the choice of activation functions are specific to GD3. Both SD3 and GD3 utilize the minimum Q-value from the double actors.

In summary, DARC, SD3, and GD3 incorporate additional hyperparameters to enhance value estimation, which increases computational complexity and can lead to performance instability if these hyperparameters are not finely tuned. In contrast, TDDR does not introduce any additional hyperparameters compared to TD3, making it easier to implement.

\begin{table}[!t] 
  \caption{Comparison of TDDR with benchmark algorithms \label{tab:components}}
  \centering
  \small
  \begin{tabular}{cccc}
    \toprule
    Algorithms &  Regularization & \makecell{Double \\ Actors} & 
     \makecell{Double \\ Critics} \\
    \midrule
    DDPG\cite{lillicrapContinuousControlDeep2019} & \ding{56} & \ding{56} & \ding{56} \\
    TD3\cite{fujimotoAddressingFunctionApproximation} & $\min$ & \ding{56} & \ding{52} \\
    DARC\cite{lyuEfficientContinuousControl2022} & weighted & \ding{52} & \ding{52} \\
    SD3\cite{panSoftmaxDeepDouble} & $\min$ & \ding{52} & \ding{52} \\
    GD3\cite{lyuValueActivationBias2023} & $\min$ & \ding{52} & \ding{52} \\
    TDDR (this work) & TD error-driven & \ding{52} & \ding{52} \\
    \bottomrule
  \end{tabular}
\end{table}

\section{Preliminaries}
\label{sec:Preliminaries}

The interaction between an agent and the environment follows a standard RL process, which is modeled as a Markov decision process (MDP). MDP is defined by the tuple $M = (S, A, P, R(s,a), \gamma)$, where $S$ is the state space, $A$ is the action space, $P$ is the state transition probability, $R(s,a)$ is the reward function, and $\gamma \in [0,1]$ is the discount factor.
The goal of RL is to find the optimal policy that maximizes cumulative rewards. Value estimation is commonly used to evaluate the effectiveness of policies, with higher value estimations generally indicating better policies\cite{lyskawaACERACEfficientReinforcement2024}. We briefly describe several relevant reinforcement learning algorithms, including DDPG, TD3, DARC, SD3, and GD3, which serve as benchmarks. Building on these, we propose and compare the new TDDR algorithm.

\subsection{Deterministic Policy Gradient}

 Silver et al. \cite{silverDeterministicPolicyGradient} introduced the deterministic policy gradient (DPG) algorithm, which uses deterministic policies instead of stochastic ones. Let $\pi_\phi$ refer to the policy with parameters $\phi$. Specifically, $\pi_\phi(s)$ denotes the action $a$ output by the policy for the given state $s$. To maximize the cumulative discounted rewards $R(s,a)$, $\phi$ is updated as $\phi \leftarrow  \phi +\alpha \nabla_\phi J(\phi)$, where $\alpha$ is the learning rate. 
 The DPG is defined as:
\begin{equation}
  \label{deqn_ex1a}
  \nabla_\phi J(\phi) = N^{-1} \sum_{s} (\nabla Q_{\theta}(s,a) \bigg|_{a=\pi_\phi(s)} \nabla_{\phi} \pi_\phi(s))
\end{equation}
where $J(\phi)$ is the objective function, typically the expected return we aim to maximize, and $\nabla Q_{\theta}(s,a) \big|{a=\pi_\phi(s)}$ represents the gradient of the action-value function $Q$ with respect to its parameters $\theta$, evaluated at $a = \pi_\phi(s)$. The gradient $\nabla Q_{\theta}(s,a)$, for updating $\theta$, is usually obtained by minimizing $N^{-1} \sum_s (y-Q_{\theta}(s,a))^2$,
where $y$ is the TD target, $y-Q_{\theta}(s,a)$ is the TD error, and $N$ represents the batch size.

Lillicrap et al. \cite{lillicrapContinuousControlDeep2019} introduced deep neural networks into DPG and proposed the DDPG algorithm. DDPG introduces the concept of target networks and utilizes a soft update approach for both policy and value parameters: 
\begin{align}
\theta^{\prime} \leftarrow \tau \theta + (1 - \tau) \theta^{\prime},\;\phi^{\prime} \leftarrow \tau \phi + (1 - \tau) \phi^{\prime} \label{targetupdate}
\end{align}
where $\theta^{\prime}$ represents the parameters of the critic target network, and $\phi^{\prime}$ represents the parameters of the actor target network. Here, $\tau \ll 1$ is a constant. DDPG is an effective RL algorithm for continuous control tasks \cite{zhangAsynchronousEpisodicDeep2021}. 
The TD target for DDPG, computed using the target networks, is: 
\begin{equation}
y = r + \gamma Q_{\theta^{\prime}} (s^{\prime}, \pi_{\phi^\prime} (s^{\prime})), \label{TD1}
\end{equation}
where $s^{\prime}$ is the next state, and the output of $\pi_{\phi^\prime} (s^{\prime})$ represents the next action, denoted as $a^{\prime}$.
 
TD3 \cite{fujimotoAddressingFunctionApproximation} is an enhanced version of DDPG that adopts a new approach of delayed policy updates to mitigate the adverse effects of update variance and bias. It also incorporates a regularization approach from machine learning by adding noise to the target policy to smooth out value estimations.

\subsection{Double Actor-Critic}

Double critics utilize two independent critic networks that generate separate value estimations without mutual influence. Originating from the concept of double Q-learning, this algorithm employs dual critics to decouple actions, thereby mitigating the overestimation issue. For example, the TD target in TD3 is defined as:
 \begin{equation} 
y=r + \gamma \min_{i=1,2}(Q_{\theta_i^{\prime}}(s^{\prime},a^{\prime})),\label{TDTD3}
 \end{equation}
where $Q_{\theta_1^{\prime}}$ and $Q_{\theta_2^{\prime}}$ represent the target networks of the double critics. The action $a^{\prime}$ is defined as: 
\begin{align}
a^{\prime} = \pi_{\phi^{\prime}}(s^{\prime})+\epsilon \label{aprime0}
\end{align}
incorporating noise $\epsilon$. The TD target \eqref{TDTD3} is used to define the TD error, $y-Q_{\theta_i}(s,a)$ for $i=1,2$, and update the double critic networks $Q_{\theta_1}$ and $Q_{\theta_2}$. TD3 demonstrates superior value estimation compared to DDPG in MuJoCo by OpenAI \cite{dhariwal2017openai}. Alongside TD3, other benchmark algorithms like DARC, SD3, and GD3 also integrate double critics.

In addition to double critics, the use of double actors $\pi_{\phi_1}$ and $\pi_{\phi_2}$, along with their 
target networks $\pi_{\phi^\prime_1}$ and $\pi_{\phi^\prime_2}$,
effectively enhances action exploration efficiency and prevents policies from settling into local optima, thereby achieving improved value estimation \cite{lyuEfficientContinuousControl2022, panSoftmaxDeepDouble}. 
The actions $a_i^{\prime}$, $i=1,2$, generated by the target networks are denoted as:
\begin{align}
a_i^{\prime} = \pi_{\phi_i^{\prime}}(s^{\prime})+\epsilon. \label{aprime}
\end{align}
Benchmark algorithms such as DARC, SD3, and GD3 also incorporate double actors.

\subsection{Regularization}

The TD target for DDPG \eqref{TD1} can be rewritten as 
 \begin{equation} 
y= r+\gamma \psi \label{TDpsi}
 \end{equation}
where 
 \begin{equation} 
\psi = Q_{\theta^{\prime}} (s^{\prime},  a^{\prime})\label{psiDDPG}
 \end{equation}
When double critics are used, the value $\psi$ for each critic following DDPG \eqref{TD1} is given as
\begin{equation}
\psi_i =   Q_{\theta_i^{\prime}} (s^{\prime}, a^\prime). \label{TD2}
\end{equation}
However, this value is not used directly. For example, it is modified in \eqref{TDTD3} to
 \begin{equation} 
\psi=\min_{i=1,2} Q_{\theta_i^{\prime}}(s^{\prime}, a^\prime). \label{hatpsiTD3}
 \end{equation}
It is noted that the same $\psi$, and hence the same TD target $y$, is used for both critics. 
The difference between \eqref{TD2} and \eqref{hatpsiTD3} acts as a regularization of the TD error.

Different forms of regularization have been proposed in the literature, varying in the construction of $\psi$.
For TD3, $\psi$ is defined in \eqref{hatpsiTD3}
where it uses the smaller Q-values from the double critics to compute $\psi$.
For DARC, $\psi$ is computed as:
\begin{align}
    \psi=& (1-\nu)\max_{j=1,2}\min_{i=1,2}Q_{\theta^{\prime}_i}(s^{\prime},a_j^\prime) \nonumber \\ &+\nu\min_{j=1,2}\min_{i=1,2}Q_{\theta^{\prime}_i}(s^{\prime},a_j^\prime)\label{hatpsiDARC}
\end{align}
where $\nu$ is the weighting coefficient. DARC computes $\psi$ by combining the Q-values of the double critics and double actors in a convex combination. In addition to the hyperparameter $\nu$ used to balance overestimation and underestimation, DARC introduces another regularization parameter, $\lambda$, which restricts the critics from differing too significantly.

For SD3 and GD3, $\psi$ is constructed as:
\begin{equation}
 \psi=\min_{j,i=1,2} Q_{\theta^{\prime}_i}(s^{\prime},a_j^\prime),\label{hatpsiSD3}
\end{equation}
which is similar to TD3 but incorporates actions from double actors.
SD3 introduces two hyperparameters. NNS influences the effect of value estimation but is not highly sensitive. In contrast, the parameter $\beta$ impacts the effectiveness and bias of value estimation: a lower $\beta$ decreases variance, while a higher $\beta$ reduces bias.
In addition to NNS and $\beta$, GD3 introduces a bias term $b$ and various activation functions, which are environment-dependent. An appropriately chosen $b$ results in softer value estimation, but if $b$ is too large, it can lead to poor performance. Activation functions, such as polynomial and exponential, are also determined by GD3.

\section{The TDDR Algorithm}
\label{sec:Proposed Method}

The TDDR algorithm adopts DAC with their respective target networks. These eight networks are denoted as
$Q_{\theta_i}$, $Q_{\theta^\prime_i}$, $\pi_{\phi_i}$, $\pi_{\phi^\prime_i}$, for $i=1,2$, as defined earlier. 
In Fig.~\ref{fig_structure}, these neural networks are simply labeled as $Q_i$,$Q_i^\prime$,$A_i$,$A_i^\prime$, for $i=1,2$, respectively. 

The core of the algorithm involves updating the eight networks. The target networks  $Q_i^\prime$ and $A_i^\prime$ are updated following conventional rules, similar to \eqref{targetupdate}, succinctly denoted as: 
\begin{align*}
\theta_i^{\prime}\leftarrow(\theta_i,\theta_i^{\prime}),\;
  \phi_i^{\prime}\leftarrow(\phi_i,\phi_i^{\prime})
\end{align*} 
for $i=1,2$. The update of each actor networks $A_i$ is performed as $\phi_i \leftarrow  \phi_i +\alpha \nabla_{\phi_i} J(\phi_i)$
where the DPG $\nabla_{\phi_i} J(\phi_i)$ is defined in \eqref{deqn_ex1a} using the corresponding critic network $Q_{\theta_i}$. 
 
Key innovations in the proposed TDDR algorithm focus on updating the critic networks $Q_i$, parameterized by $\theta_i$. The first innovation involves adopting the clipped double Q-learning (CDQ) approach, which utilizes double target actors to evaluate actions, resulting in CDQ based on double actors (DA-CDQ). Regularizing the TD error plays a crucial role in updating the double critic networks. The second innovation introduces a novel critic regularization architecture (CRA), which is driven by the TD error of their target networks. These two features, along with a comparative analysis against benchmark algorithms, are elaborated upon in the subsequent sections.

Lastly, it is noteworthy that TDDR also utilizes the cross-update architecture employed in DARC, where only one pair, either $(Q_1,A_1)$ or $(Q_2,A_2)$, is updated at each step.

\begin{figure*}[!t]
  \centering
  \vspace{-2mm}
  \includegraphics[width=0.76\textwidth]{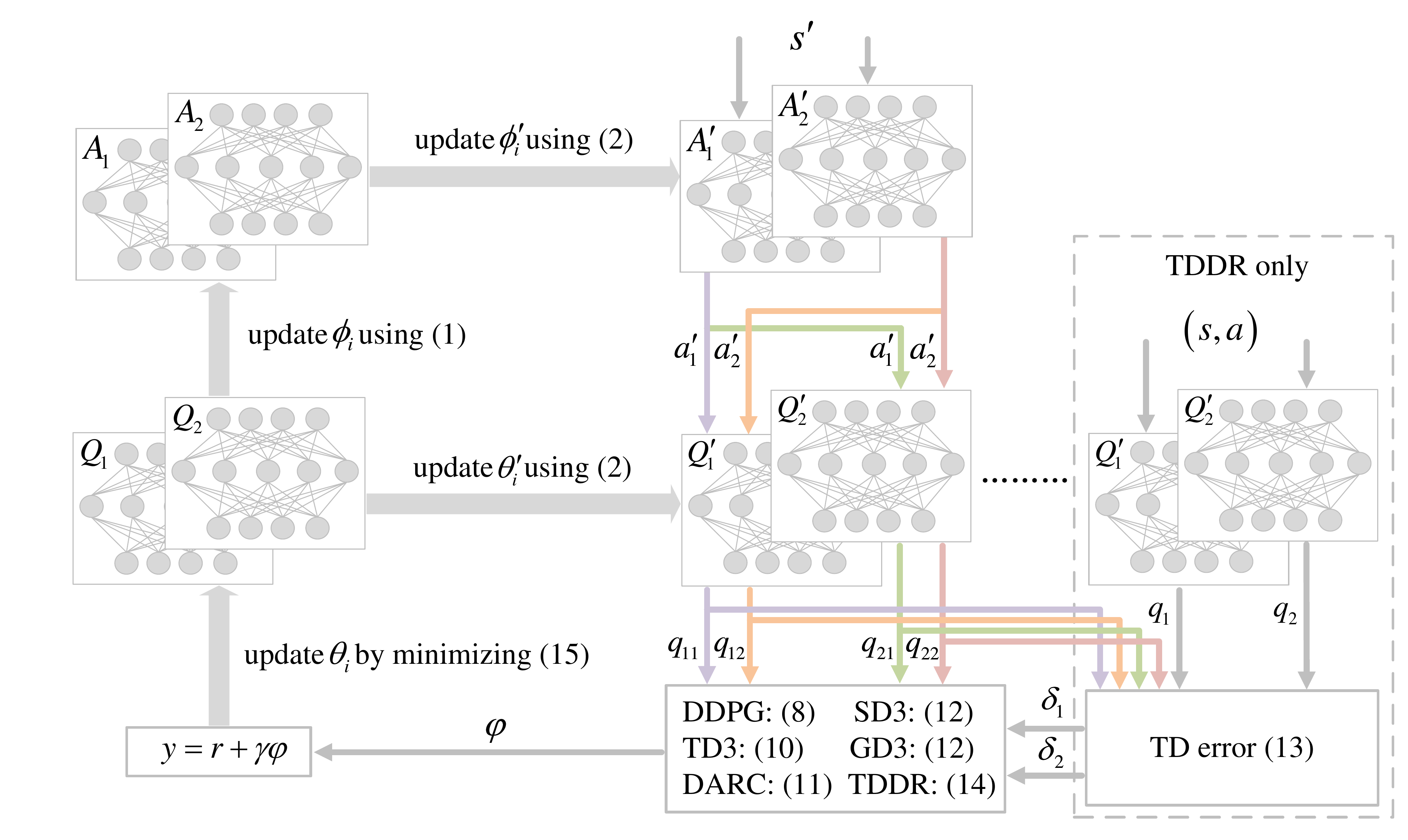}
  \caption{Architecture of TDDR and the benchmark algorithms: $q_{ij} = Q_{\theta_i'}(s', a_j')$ and $q_i = Q_{\theta_i'}(s, a)$; the duplication of $Q_1'/Q_2'$ indicates that the same networks are used with different inputs; the action $a_i'$ is generated by $A_i'$ following \eqref{aprime}. }
  \label{fig_structure}
  \end{figure*}

\subsection{Double Actors with CDQ}

The proposed algorithm adopts the CDQ approach, similar to TD3, by incorporating noise during the action sampling process as defined in \eqref{aprime} for each target actor $i=1,2$. Here, $\epsilon\sim \operatorname{clip}({\cal N}(0,\sigma),-c,c)$ represents Gaussian noise with mean $0$ and variance $\sigma$, clipped to the range of $[-c,c]$ for a positive constant $c$. This clipping effectively limits the noise magnitude, ensuring that actions remain within a reasonable range and preventing excessive perturbations.

The DA-CDQ method, inspired by TD3, employs the min operator to clip both value estimations during the update process. In contrast to TD3's single actor, DA-CDQ utilizes double actors, where each actor network independently selects actions. These actions are then evaluated by double critics to assess their value. Importantly, each actor network can autonomously learn and refine its own policy. DA-CDQ enhances exploration capabilities and effectively prevents actions from becoming trapped in local optima, which can occur with a single actor setup like TD3. This approach facilitates the discovery of potentially optimal policies \cite{lyuEfficientContinuousControl2022}. Using a single actor, as in TD3, results in identical training targets for the double critics, potentially limiting the advantages they offer \cite{panSoftmaxDeepDouble}.
 
\subsection{Critic Regularization Architecture}

Each actor target generates an action $a_i^{\prime}$ as in \eqref{aprime}, which, together with the target networks of double critics, can be used to define the TD error of the critic targets $Q_i^{\prime}$ as follows:
\begin{equation} \label{deltai}
 \delta_i=r+\gamma \min_{j=1,2}(Q_{\theta_j^{\prime}}(s^{\prime},a_i^{\prime}))-\min_{j=1,2}(Q_{\theta_j^{\prime}}(s,a))  
\end{equation}
for $i = 1, 2$.
For each action $a_i^{\prime}$, both critic target networks $Q_j^{\prime}$, $j = 1, 2$, are used to estimate the values for the next state $s^{\prime}$. The smaller value is used to define the TD target $r + \gamma \min_{j=1,2} (Q_{\theta_j^{\prime}} (s^{\prime}, a_i^{\prime}))$.
Correspondingly, the two critic target networks generate two action-values for $(s, a)$, and the smaller value $\min_{j=1,2} (Q_{\theta_j^{\prime}}(s, a))$ is utilized as the actual action-value. The TD target and the actual action-value define the TD error in \eqref{deltai} for each actor target network.

It is important to note that these TD errors in \eqref{deltai} are not for the critics $Q_i$ nor are they used to update $\theta_i$ directly. However, these TD errors drive the calculation of the regularization of the TD errors of $Q_i$. This TD error-driven regularization is a novel feature compared to the existing regularizations introduced earlier. The specific computation of $\psi$ for TDDR is given as
\begin{equation} \label{psiTDDR}
\psi= \begin{aligned}
  \begin{cases}
\min_{i=1,2}(Q_{\theta_i^{\prime}}(s^{\prime},a_1^{\prime})), & \text{if } \big|\delta_1\big|\leq\big|\delta_2\big|\\
\min_{i=1,2}(Q_{\theta_i^{\prime}}(s^{\prime},a_2^{\prime})), & \text{if } \big|\delta_1\big|> \big|\delta_2\big| 
  \end{cases}.
  \end{aligned}
  \end{equation}
In the existing calculations of $\psi$ in \eqref{hatpsiTD3}, \eqref{hatpsiDARC}, and \eqref{hatpsiSD3}, the smaller action-value of the double critic target networks, evaluated either on a single actor target or double actor targets, is directly used. In contrast, in \eqref{psiTDDR}, the smaller action-value, evaluated on the selected actor target network that gives the smaller TD error ($|\delta_1|$ or $|\delta_2|$), is used.

Once $\psi$ is determined in \eqref{psiTDDR}, the actual TD error for the critics $Q_i$ is calculated as in \eqref{TDpsi}. The gradient $\nabla Q_{\theta_i}(s,a)$, used for updating $\theta_i$, is obtained by minimizing 
\begin{equation}
N^{-1}\sum_s (y-Q_{\theta_i}(s,a))^2. \label{NyQ}
\end{equation} 

Using DAC to compute $\psi$ as in \eqref{psiTDDR} in TDDR offers several advantages. First, it avoids overestimation by selecting the smaller value based on the comparison between $|\delta_1|$ and $|\delta_2|$. This ensures that the TD target is calculated based on the action-value generated by either target policy $A_1^{\prime}$ or $A_2^{\prime}$, reducing the risk of overestimation.

Additionally, it balances exploration and exploitation. By using two Q-values to calculate the TD target, the approach encourages a balance between exploring new actions and exploiting known ones. When comparing $|\delta_1|$ and $|\delta_2|$, if the current policy is $A_1^{\prime}$, then $A_1^{\prime}$ is exploited while $A_2^{\prime}$ is explored, and vice versa. This dynamic helps in discovering potentially optimal policies while maintaining robustness in the learning process.

TDDR is summarized in Algorithm~\ref{alg:TDDR}.

\begin{algorithm}[tb]
    \caption{\underline{TD} Error-\underline{D}riven \underline{R}egularization (TDDR)}
    \label{alg:TDDR}
    \begin{algorithmic}[1] 
    \STATE Initialize critic networks $Q_{\theta_1},Q_{\theta_2}$ and actor networks $\pi_{\phi_1},\pi_{\phi_2}$ with random parameters $\theta_1,\theta_2,\phi_1,\phi_2$ 
    \STATE Initialize target networks $\theta_1^{\prime}$, $\theta_2^{\prime}$, $\phi_1^{\prime}$, $\phi_2^{\prime}$
    \STATE Initialize replay buffer $\mathfrak{R}$
    \FOR{$\Theta$ = 1 to $N$}
    \STATE Select action $a$ with $\max_i\max_jQ_{\theta_i}(s,\pi_{\phi_j}(s))$ added $\epsilon\sim\mathcal{N}(0,\sigma)$
    \STATE Execute action $a$ and observe reward $r$, new state $s^{\prime}$ and done flag $d$
    \STATE Store transitions in the experience replay buffer, i.e., $(s_\Theta,a_\Theta,r_\Theta,s_{\Theta+1}^{\prime},d_\Theta)$
    \FOR{$i = 1,2$}
    \STATE Sample $\{(s_{\Theta},a_{\Theta},r_{\Theta},s_{\Theta+1}^{\prime},d_{\Theta})\}_{\Theta=1}^{N}\sim\mathfrak{R}$\
    \STATE Calculate $a_i$ with \eqref{aprime}
    \STATE Calculate $\delta_i$ with \eqref{deltai}
    \STATE Calculate $y$ with \eqref{TDpsi}, \eqref{psiTDDR}
    \STATE Update critic $\theta_{i}$ by minimizing \eqref{NyQ}
    \STATE Update actor $\phi_{i}$ with policy gradient: \\ $N^{-1}\sum_s(\nabla Q_{\theta_i}(s,a)|_{a=\pi_{\phi_i}(s)}\nabla_{\phi_i(s)}\pi_{\phi_i}(s))$\
    \STATE Update target networks: \\ $\theta_i^{\prime}\leftarrow(\theta_i,\theta_i^{\prime}),\ \phi_i^{\prime}\leftarrow(\phi_i,\phi_i^{\prime})$
    \ENDFOR
    \ENDFOR
    \end{algorithmic}
    \end{algorithm}

\subsection{Comparative Analysis of Algorithms}

The architecture of TDDR is depicted in Fig.~\ref{fig_structure}. Benchmark algorithms can be represented similarly: DDPG consists of one critic and one actor, TD3 incorporates double critics with one actor, and DARC, SD3, and GD3 have DAC architectures like TDDR. Using this figure, the specific differences between TDDR and these benchmark algorithms are analyzed below.

  When the eight neural networks reduce to four, with $\theta_i$, $\theta_i^{\prime}$, $\phi_i$, and $\phi_i^{\prime}$ for $i=1,2$ being replaced by $\theta$, $\theta^{\prime}$, $\phi$, and $\phi^{\prime}$, respectively, TDDR simplifies to DDPG, which is the simplest mode among these six algorithms.
With only $\phi_i$ and $\phi_i^{\prime}$ for $i=1,2$ replaced by $\phi$ and $\phi^{\prime}$, respectively, the remaining six neural networks of TDDR constitute TD3.

Given that DARC, SD3, and GD3 share the same network architecture as TDDR, their distinctions are subtle but impactful in terms of performance. The first distinction is the use of CDQ in TDDR, which is absent in the other algorithms. The second distinction is in the definition of the TD error-driven $\psi$ in \eqref{psiTDDR}, which differs significantly from those in other algorithms, such as \eqref{hatpsiDARC} and \eqref{hatpsiSD3}. This difference is highlighted by the ``TDDR only'' block in the figure.

The third distinction is that TDDR does not introduce additional hyperparameters beyond those in TD3, whereas DARC, SD3, and GD3 introduce two, two, and four additional hyperparameters, respectively. While hyperparameters can increase the complexity of the algorithms and the tuning workload, TDDR simplifies the design and implementation process by avoiding this additional complexity.

For instance, the critics in DARC are regularized to be close according to \eqref{hatpsiDARC}, ensuring they are not isolated. While the mutually regularized double critics in DARC perform well in value function estimation, their performance is significantly influenced by the regularization coefficient $\nu$. Appropriate regularization coefficients can lead to better value estimation, whereas test results show that unsuitable regularization coefficients can significantly increase the standard deviation of the results, reducing DARC's stability. In contrast, the mutually independent double critics structure in TDDR is not affected by hyperparameter selection. In TDDR, the critics are fully utilized, with each critic responsible for its corresponding counterpart actor, and both critics contributing to policy improvement.

\section{Convergence Analysis}
\label{sec:Convergence Analysis}

Convergence analysis has been studied in double Q-learning \cite{hasselt2010double}, which can be formulated as follows:
\begin{align} 
&Q_{t+1}^A(s_t,a_t) = Q_t^A(s_t,a_t) + \alpha_t(s_t,a_t) \times \nonumber\\
& \quad (r_t+  \gamma Q_t^B(s_{t+1}, a_1^*) -Q_t^A(s_t,a_t)) \nonumber\\
&Q_{t+1}^B(s_t,a_t) = Q_t^B(s_t,a_t) + \alpha_t(s_t,a_t) \times\nonumber\\
& \quad (r_t+  \gamma  Q_t^A(s_{t+1}, a_2^*) -Q_t^B(s_t,a_t)) \label{policydQ} \end{align}
where $Q_t^A(s_t,a_t)$ and $Q_t^B(s_t,a_t)$ represent the updates of double Q-values, and $r_t = R(s_t, a_t)$ is the reward for the pair $(s_t, a_t)$. The actions $a_1^*$ and $a_2^*$ are defined as:
\begin{align*}
a_1^* = \arg\max_{a}Q_t^A (s_{t+1}, a),\;
a_2^* = \arg\max_{a}Q_t^B (s_{t+1}, a).\end{align*} 

It is noted that $r_t + \gamma Q_t^B(s_{t+1}, a_1^*)$ is the TD target for $Q_t^A(s_t, a_t)$, and vice versa for swapping $A$ and $B$. Treating the critics $Q_{\theta_1}(s, a)$ and $Q_{\theta_2}(s, a)$ as $Q_t^A(s, a)$ and $Q_t^B(s, a)$, respectively, the TD target in \eqref{TDpsi} is $r_t + \gamma \psi_t$, explicitly depending on $t$. 
In the definitions of $\psi$ in \eqref{hatpsiTD3}, \eqref{hatpsiDARC}, \eqref{hatpsiSD3}, or \eqref{psiTDDR}, $\psi$ cannot be simply expressed as $Q_t^B(s_{t+1}, a_1^*)$ or $Q_t^A(s_{t+1}, a_2^*)$ as in \eqref{policydQ}. The first difference is that $\psi$ is calculated from the critic and actor target networks, while $Q_t^B(s_{t+1}, a_1^*)$ or $Q_t^A(s_{t+1}, a_2^*)$ are the directly calculated optimal values. The influence of adding target networks to convergence is too complicated to analyze here.

The second difference is that $\psi$ is not $Q_t^B(s_{t+1}, a_1^*)$ or $Q_t^A(s_{t+1}, a_2^*)$, but rather the smaller value generated by the critics, denoted as $\min_{i=1,2} Q_{\theta_i^{\prime}}$ in all the aforementioned definitions of $\psi$. To capture this difference in all the aforementioned algorithms, including DARC, SD3, GD3, and TDDR, the formulation in \eqref{policydQ} is modified to
\begin{align} 
&Q_{t+1}^A(s_t,a_t) = Q_t^A(s_t,a_t) + \alpha_t(s_t,a_t)(r_t+ \nonumber\\
&\quad  \gamma \min\{Q_t^B(s_{t+1}, a_{i=1,2}^*), Q_t^A(s_{t+1}, a_{i=1,2}^*)\}-Q_t^A(s_t,a_t)) \nonumber \\
&Q_{t+1}^B(s_t,a_t) = Q_t^B(s_t,a_t) + \alpha_t(s_t,a_t)(r_t+  \nonumber\\
&\quad  \gamma \min\{Q_t^B(s_{t+1}, a_{i=1,2}^*), Q_t^A(s_{t+1}, a_{i=1,2}^*)\}-Q_t^B(s_t,a_t)). \label{policy2} \end{align}
Here, $a^*_{i=1,2}$ can be either $a^*_1$ or $a^*_2$, depending on the specific algorithm. For instance, in TDDR, $a^*_{i=1,2}=a^*_1$ when $|\delta_1|\leq|\delta_2|$, and $a^*_{i=1,2}=a^*_2$ otherwise, where $\delta_i$ defined in \eqref{deltai} is modified to 
\begin{align*}
\delta_{i=1,2} =& r_t + \gamma \min\{Q_t^B(s_{t+1}, a_{i=1,2}^*), Q_t^A(s_{t+1}, a_{i=1,2}^*)\} \\
& - \min\{Q_t^B(s_t, a_t), Q_t^A(s_t, a_t)\}
\end{align*}
It is worth noting that the convergence analysis of \eqref{policy2} does not necessarily rely on a particular selection of $a^*_{i=1,2}$.

In \eqref{policy2}, $Q_t^A(s_t, a_t)$ and $Q_t^B(s_t, a_t)$ are updated using the same TD target, denoted as
\begin{align*}
    T_t =r_t+\gamma\min\{Q_t^B(s_{t+1},a_{i=1,2}^*), Q_t^A(s_{t+1}, a_{i=1,2}^*)\},
\end{align*}
both with $a^*_1$ or $a^*_2$. The TD errors are denoted as 
 \begin{align*}
E_t^A=T_t - Q_t^A(s_t, a_t),\;  E_t^B=T_t -Q_t^B(s_t, a_t).\end{align*}
However, in \eqref{policydQ}, $Q_t^A(s_t, a_t)$ and $Q_t^B(s_t, a_t)$ are updated using different TD targets, one with $a^*_1$ and the other with $a^*_2$. This distinction is significant, as it facilitates the convergence of \eqref{policy2} in subsequent analysis.

Next, our aim is to analyze the convergence of \eqref{policy2} to the optimal value $Q_t^*(s_t,a_t)$ under two updating patterns: random updating and simultaneous updating. In the random updating pattern, either $Q_t^A(s_t, a_t)$ or $Q_t^B(s_t, a_t)$ is chosen for update with a probability of $0.5$. In contrast, in the simultaneous updating pattern, both Q-values are updated at the same time. The convergence analysis relies on the following technical lemma. The proof of this lemma can be found in \cite{singh2000convergence}, building on earlier work in \cite{jaakkola1993convergence, bertsekas1995dynamic}.

\begin{lemma} \label{lemma}
	Consider a stochastic process $(\zeta_t, \Delta_t, F_t)$, $t \geq 0$, where $\zeta_t$, $\Delta_t$, and $F_t$ satisfy:
\begin{align} \label{policy3}
		\Delta_{t+1}(x_t)=(1-\zeta_t(x_t)) \Delta_t(x_t)+\zeta_t(x_t) F_t(x_t)
\end{align}
for  $x_t \in X$ and $t\geq 0$. Let $P_t$ be a sequence of increasing $\sigma$-fields such that $\zeta_0$ and $\Delta_0$ are $P_0$-measurable, and $\zeta_t$, $\Delta_t$ and $F_{t-1}$ are $P_t$-measurable, $t=1, 2, \ldots$. Assume that the following hold: 

\begin{enumerate}

\item The set $X$ is finite;
	
\item  $\zeta_t(x_t) \in [0,1]$, $\sum_t \zeta_t(x_t)=\infty$, $\sum_t(\zeta_t(x_t))^{2}<\infty$ with probability 1 and $\forall x \neq x_t: \zeta_t(x)=0$;
	
\item  $\| \bE[F_t | P_t] \| \leq \kappa\| \Delta_t\|+c_t$, where $\kappa \in [0,1)$ and $c_t$ converges to 0 with probability 1;
	
\item  $\Var[F_t(x_t) | P_t] \leq K(1+\kappa \| \Delta_t\|)^{2}$, where $K$ is some constant. 

\end{enumerate}
	Then $\Delta_t$ converges to 0 with probability 1. Here, $\|\cdot\|$ denotes a maximum norm.
\end{lemma}

The main result is stated in the following theorem. 
 
\begin{theorem}\label{theorem123}
The double Q-values, $Q^A$ and $Q^B$, updated using \eqref{policy2}, either through random or simultaneous updates, converge to the same optimal $Q^*$ with probability 1 if the following conditions hold:

\begin{enumerate}
\item Each state action pair is sampled an infinite number of times;
	
\item The MDP is finite;
	
\item $\gamma \in [0, 1)$;
	
\item Q values are stored in a lookup table;
	
\item Both $Q^A$ and $Q^B$ receive an infinite number of updates;
	
\item The learning rates satisfy $\alpha_t(s, a) \in[0,1], \sum_t \alpha_t(s, a)=\infty, \sum_t(\alpha_t(s, a))^{2}<\infty$ with probability 1 and $\alpha_t(s,a)=0, \forall(s, a) \neq(s_t, a_t)$;
	
\item $\Var[R(s, a)]<\infty, \forall s, a$;

\item The TD errors $E_t^A$ and $E_t^B$ are bounded.

\end{enumerate}
 
\end{theorem}

\begin{proof} We begin by providing the proof for the case of random updating, which is divided into two steps. In the first step, let $\Delta_t^{BA}(s_t,a_t)=Q_t^B(s_t,a_t) - Q_t^A(s_t,a_t)$. Under the random updating pattern, either $Q_t^A(s_t, a_t)$ or $Q_t^B(s_t, a_t)$ is updated according to \eqref{policy2}, each with a probability of $0.5$. Specifically, when updating $Q_t^A(s_t,a_t)$, we have
\begin{align} \label{policy5a}
    &\Delta_{t+1}^{B A}(s_t,a_t) 
    \nonumber \\=&Q_{t+1}^B(s_t,a_t)-Q_{t+1}^A(s_t,a_t)
    \nonumber \\=&Q_t^B(s_t,a_t)-Q_{t+1}^A(s_t,a_t)
    \nonumber \\=&Q_t^B(s_t,a_t)-Q_t^A(s_t,a_t) - \alpha_t(s_t,a_t)E_t^A
    \nonumber \\=&\Delta_t^{B A}(s_t,a_t)-\alpha_t(s_t, a_t) E_t^A. 
\end{align}
Similarly, when updating $Q_t^B(s_t,a_t)$, we have
\begin{equation} \label{policy5b}
\Delta_{t+1}^{B A}(s_t,a_t) =\Delta_t^{B A}(s_t,a_t)+\alpha_t(s_t, a_t) E_t^B.
\end{equation}
Denote
\begin{align*}
    F_t^{BA1}(s_t,a_t) =-2E_t^A+\Delta_t^{B A}(s_t,a_t)
\end{align*}
and
\begin{align*}
    F_t^{BA2}(s_t,a_t) =2E_t^B+\Delta_t^{B A}(s_t,a_t).
\end{align*}

Equations \eqref{policy5a} and \eqref{policy5b} can be expressed as
\begin{align} 
\Delta_{t+1}^{B A}(s_t,a_t) =\left(1- \frac{1}{2}\alpha_t(s_t,a_t)\right) \Delta_t^{B A}(s_t,a_t) \nonumber \\
+\frac{1}{2} \alpha_t(s_t, a_t) F_t^{BA}(s_t,a_t) \label{policy5-1}
\end{align}
where $F_t^{BA}(s_t,a_t)$ takes the value of $F_t^{BA1}(s_t,a_t)$ or $F_t^{BA2}(s_t,a_t)$ with probability $0.5$.

We apply Lemma~\ref{lemma} to the process \eqref{policy5-1}, where $x_t = (s_t, a_t)$, $\Delta_{t}(x_t) = \Delta_{t}^{BA}(s_t, a_t)$, $\zeta_t(x_t) = \frac{1}{2} \alpha_t(s_t, a_t)$, and $F_t(x_t) = F_t^{BA}(s_t, a_t)$. Let $P_t = \{Q_0^A, Q_0^B, s_0, \dots, s_t, a_0, \dots, a_t, r_0, \dots, r_{t-1} \}$. Conditions 1) and 2) are clearly satisfied. Next, we verify that $\bE[F_t^{BA}(s_t, a_t)]$ and $\Var[F_t^{BA}(s_t, a_t)]$ meet conditions 3) and 4), where the expectation and variance are calculated for the random variable given $P_t$.

Using \eqref{EABC}, we have
\begin{align*}
    & \bE[F_t^{BA}(s_t,a_t)] \nonumber\\ 
    = & \frac{1}{2} \left[ \bE[F_t^{BA1}(s_t,a_t)] + \bE[F_t^{BA2}(s_t,a_t)] \right] \nonumber\\
    = & \bE[(T_t - Q_t^B(s_t,a_t))-(T_t - Q_t^A(s_t,a_t)) \nonumber\\
    & + \Delta_t^{B A}(s_t,a_t)] = 0.
\end{align*} 
This verifies condition 3). It also implies $\bE[F_t^{BA1}(s_t,a_t)]=-\bE[F_t^{BA2}(s_t,a_t)]$. 
Next, using \eqref{varABC}, we have
\begin{align*}
    & \Var[F_t^{BA}(s_t,a_t)] \nonumber \\ 
    = & \frac{1}{2}\Var[F_t^{BA2}(s_t,a_t)] + \frac{1}{2}\Var[F_t^{BA1}(s_t,a_t)] \nonumber \\
    & + \frac{1}{4}(\bE[F_t^{BA2}(s_t,a_t)]   - \bE[F_t^{BA1}(s_t,a_t)])^2 \nonumber \\
    = & \frac{1}{2} \bigg[ \bE[(F_t^{BA2}(s_t,a_t))^2]  - (\bE[F_t^{BA2}(s_t,a_t)])^2 \bigg] \nonumber \\
    & +\frac{1}{2} \bigg[ \bE[(F_t^{BA1}(s_t,a_t))^2]  - (\bE[F_t^{BA1}(s_t,a_t)])^2 \bigg] \nonumber \\
    & +  (\bE[F_t^{BA2}(s_t,a_t)])^2 \nonumber \\
    = & \frac{1}{2} \bigg[ \bE[(F_t^{BA2}(s_t,a_t))^2]  + \bE[(F_t^{BA1}(s_t,a_t))^2]    \bigg]   \nonumber \\   
    = & \frac{1}{2} \bigg[ \bE[(2E_t^B+\Delta_t^{B A}(s_t,a_t))^2] \nonumber \\
    & + \bE[(-2E_t^A+\Delta_t^{B A}(s_t,a_t))^2] \bigg] \nonumber \\
    = & \bE[2(E_t^B)^2 + 2(E_t^A)^2 - (\Delta_t^{BA}(s_t,a_t))^2] \nonumber \\
    \leq &K(1 + \kappa\|\Delta_t^{BA}(s_t,a_t)\|)^2
\end{align*}
for some constants $K$ and $\kappa$. This verifies condition 4). By applying Lemma~\ref{lemma}, we conclude that $\Delta_t^{BA}(s_t, a_t) = Q_t^B(s_t, a_t) - Q_t^A(s_t, a_t)$ converges to 0 with probability 1, completing the proof of the first step.

\medskip

In the second step, let $\Delta_t=Q_t^A-Q^*$. Then, using \eqref{policy2}, when updating $Q_t^A(s_t,a_t)$, we have
\begin{align} \label{DeltaA1}
& \Delta_{t+1}(s_t,a_t)= Q_{t+1}^A(s_t,a_t)-Q^*(s_t,a_t). 
\end{align} 
Similarly, when updating $Q_t^B(s_t,a_t)$, we have
\begin{align} \label{DeltaA2}
& \Delta_{t+1}(s_t,a_t)= \Delta_t(s_t,a_t).  
\end{align} 
Equations \eqref{DeltaA1} and \eqref{DeltaA2} can be rewritten as
\begin{align} \label{DeltaA}
 \Delta_{t+1}(s_t,a_t) = & (1-\alpha_t(s_t,a_t))\Delta_t(s_t,a_t) \nonumber\\
 & +\alpha_t(s_t,a_t) F_t(s_t,a_t)
\end{align} 
where
\begin{align}
    F_t(s_t,a_t)=&T_t -Q_t^*(s_t, a_t) \label{Ft1}
\end{align}
or 
\begin{align}
    F_t(s_t,a_t)= \Delta_t(s_t,a_t), \label{Ft2}
\end{align}
each with probability $0.5$.

We apply Lemma~\ref{lemma} to the process described by \eqref{DeltaA}, with $x_t = (s_t, a_t)$ and $\zeta_t(x_t) = \alpha_t(s_t, a_t)$. As in the first step, we need to verify conditions 3) and 4).

Let 
\begin{align*}
F_t^Q (s_t,a_t) = r_t + \gamma Q_t^A(s_{t+1}, a_1^*) - Q_t^*(s_t, a_t)
\end{align*}
denote the TD error under standard Q-learning.
It is well-known that $\bE[F_t^Q(s_t,a_t)] \leq \gamma\|\Delta_t\|$ \cite{jaakkola1993convergence}.
Additionally, let
\begin{align*}
    c_t =& \min\{Q_t^B(s_{t+1},a_{i=1,2}^*), Q_t^A(s_{t+1}, a_{i=1,2}^*)\} \\& - Q_t^A(s_{t+1}, a_1^*)
\end{align*}
denote a term that converges to 0 with probability 1, as established in the first step.

Using \eqref{EABC}, we have \begin{align*}
& \bE[F_t(s_t,a_t)] \nonumber\\ 
= & \frac{1}{2} \bigg[ \bE[T_t -Q_t^*(s_t, a_t)] + \bE[\Delta_t(s_t,a_t)] \bigg] \nonumber\\
= & \frac{1}{2} \bigg[ \bE[F_t^Q(s_t, a_t) + \gamma c_t]   + \bE[\Delta_t(s_t,a_t)] \bigg] \\
\leq & \frac{1}{2}(\gamma+1) \|\Delta_t\| +\frac{\gamma}{2}c_t.
\end{align*} 
This verifies condition 3).
Next, using \eqref{varABC}, we have
\begin{align*}
    & \Var[F_t(s_t,a_t)] \nonumber \\ 
    = & \frac{1}{2}\Var[F_t^Q(s_t,a_t) + \gamma c_t]  + \frac{1}{2}\Var[\Delta_t(s_t,a_t)] \nonumber \\
    & + \frac{1}{4}(\bE[F_t^Q(s_t,a_t) + \gamma c_t]  - \bE[\Delta_t(s_t,a_t)])^2 \nonumber \\
 \leq &   \Var[F_t^Q(s_t,a_t)]  +   \Var[\gamma c_t]  \\
  & + \frac{1}{4}\left( (\gamma +1) \|\Delta_t\| + \bE[\gamma c_t]\right)^2. 
\end{align*}
Note that $\Var[\Delta_t(s_t,a_t)]=0$.
It is well-known that $\Var[F_t^Q(s_t,a_t)] \leq C(1+\kappa \| \Delta_t\|)^{2}$ \cite{jaakkola1993convergence}. Additionally, $\gamma c_t$ converges to 0 with probability 1, as established in the first step. This verifies condition 4). By applying Lemma~\ref{lemma}, we conclude that $\Delta_t(s_t, a_t)$ converges to 0 with probability 1, completing the proof of the second step.

The above two steps together prove that $Q^A$ and $Q^B$, updated using \eqref{policy2} with random updating, converge to the same optimal $Q^*$ with probability 1.

The proof for simultaneous updating is similar, with some key differences explained below. In the first step, we have
\begin{align*} 
    & \Delta_{t+1}^{B A}(s_t, a_t) \nonumber \\
    = & Q_{t+1}^B(s_t, a_t)-Q_{t+1}^A(s_t, a_t) \nonumber \\
    = & Q_t^B(s_t, a_t) + \alpha_t(s_t, a_t)E_t^B  - Q_t^A(s_t, a_t)-\alpha_t(s_t, a_t)E_t^A \nonumber \\
    = & (1-\alpha_t(s_t, a_t)) \Delta_t^{B A}(s_t, a_t) 
\end{align*}
which resembles the form of \eqref{policy5-1}, but with $(1/2) \alpha_t(s_t, a_t)$ replaced by $\alpha_t(s_t, a_t)$ and $F_t^{BA}(s_t, a_t) = 0$. The remaining proof follows straightforwardly.

In the second step, \eqref{DeltaA} applies where $F_t(s_t, a_t)$ always takes the value given in \eqref{Ft1}, rather than \eqref{Ft2}. Consequently, the calculations for $\bE[F_t(s_t, a_t)]$ and $\Var[F_t(s_t, a_t)]$ are slightly modified as follows: 
  \begin{align*}
& \bE[F_t(s_t,a_t)] \nonumber\\ 
= &  \bE[T_t -Q_t^*(s_t, a_t)] =  \bE[F_t^Q(s_t, a_t) + \gamma c_t]  \\
\leq &\gamma \|\Delta_t\| +\gamma c_t.
\end{align*} 
 and 
 \begin{align*} 
    & \Var[F_t(s_t,a_t)]  = \Var[F_t^Q(s_t,a_t) + \gamma c_t]  \\
 \leq &  2 \Var[F_t^Q(s_t,a_t)]  + 2  \Var[\gamma c_t].  
\end{align*}
Therefore, conditions 3) and 4) are also verified, and the remaining proof follows accordingly.  \end{proof}

\section{Experiments}
\label{sec:experiments}

We conducted extensive experiments on nine continuous control tasks using MuJoCo and Box2d environments from OpenAI Gym \cite{brockmanOpenAIGym2016, todorovMuJoCoPhysicsEngine2012}. Table~\ref{table.tasks} lists the state and action dimensions of these environments. Our benchmark algorithms include DDPG \cite{lillicrapContinuousControlDeep2019}, TD3 \cite{fujimotoAddressingFunctionApproximation}, DARC \cite{lyuEfficientContinuousControl2022}, SD3 \cite{panSoftmaxDeepDouble}, and GD3 \cite{lyuValueActivationBias2023}.

\begin{table}[h]  
  \caption{State and action dimensions of nine control tasks\label{tab:table2}}
  \centering
  \small
  \begin{tabular}{cccccc}
    \toprule
    Environment & State & Action \\
    \midrule
    Ant-v2 & 111 & 8 \\
    HalfCheetah-v2 & 17 & 6 \\
    Hopper-v2 & 11 & 3 \\
    Walker2d-v2 & 17 & 6 \\
    Reacher-v2 & 11 & 2 \\
    InvertedDoublePendulum-v2 & 11 & 1 \\
    InvertedPendulum-v2 & 4 & 1 \\
    BipedalWalker-v3 & 24 & 4 \\
    LunarLanderContinuous-v2 & 8 & 2 \\     
    \bottomrule
  \end{tabular}
  \label{table.tasks}
\end{table}

Given the emphasis on reproducibility \cite{hendersonDeepReinforcementLearning2018, nagarajanDeterministicImplementationsReproducibility2019} and the importance of effective comparisons \cite{colasHowManyRandom2018}, we conducted our experiments using five random seeds and network initializations to ensure fairness in our comparisons. Additionally, we utilized the default reward function and environment settings without modification to maintain a level playing field. Each of the nine selected environments was run for 1 million steps, with evaluations conducted every 5,000 steps using the average reward from 10 episodes for each evaluation.  The batch size is set to 128.

The experiment is structured into three parts. The first part compares TDDR with DDPG and TD3, which have different AC architectures. The second part compares TDDR with DARC, SD3, and GD3, all of which share the same DAC architecture. The third part consists of ablation discussion about the impact of different components of TDDR.

In the second part, since TDDR does not introduce additional hyperparameters unlike DARC, SD3, and GD3, the performance comparison is divided into two scenarios. The first scenario involves configuring the three benchmark algorithms with hyperparameters that yield relatively better performance, while the second scenario involves setting them with hyperparameters that result in relatively poorer performance. The values of the hyperparameters are listed in Table~\ref{table.hyperpara}. This aims to validate the impact of hyperparameters on performance metrics such as standard deviation. It should be noted that, for a fair comparison, the hyperparameters of DARC, SD3, and GD3 were selected based on the results of ablation experiments from their original papers. Here, ``additional hyperparameters'' refers to those beyond those of TD3. 

\begin{table}[h!] 
  \caption{Hyperparameters setup\label{tab:table4}}
  \centering
  \small
  \begin{tabular}{cc}
    \toprule
    Hyperparameter & Value (better / worse) \\
    \midrule
    \textbf{DARC} \\
    regularization parameter $\lambda$ & 0.005 \textbf{/} 0.1 \\
    weighting coefficient $\nu$ & 0.1 \textbf{/} 0.5 \\
    \textbf{SD3} \\
    NNS & 50 \textbf{/} 2 \\
    parameter $\beta$ & 0.001 \textbf{/} 0.5 \\
    \textbf{GD3} \\
    NNS & 50 \textbf{/} 2 \\
    parameter $\beta$ & 0.05 \textbf{/} 0.5 \\
    bias term $b$ & 2 \textbf{/} 2 \\
    activation functions & 
    environment dependent\\
   \bottomrule
  \end{tabular}
  \label{table.hyperpara}
\end{table}

\subsection{Comparison with DDPG and TD3}

The experimental results are plotted in Fig.~\ref{fig:baseline}, where the solid curves depict the mean across evaluations and the shaded region represents one standard deviation over five runs. A sliding window of five was applied to achieve smoother curves.
This same representation is also used in Figs.~\ref{fig:better} and \ref{fig:worse}.

Table~\ref{table:baseline} displays the average returns from five random seeds for the last ten evaluations, with the maximum value for each task highlighted in bold. This formatting is also applied in Tables~\ref{table.better} and \ref{table.worse}. 

The comparison here is straightforward. Both Fig.~\ref{fig:baseline} and Table~\ref{table:baseline} clearly demonstrate that TDDR outperforms DDPG and TD3 across all environments.

\begin{figure*}
  \centering
      \subcaptionbox{}{\includegraphics[width = 0.32\textwidth]{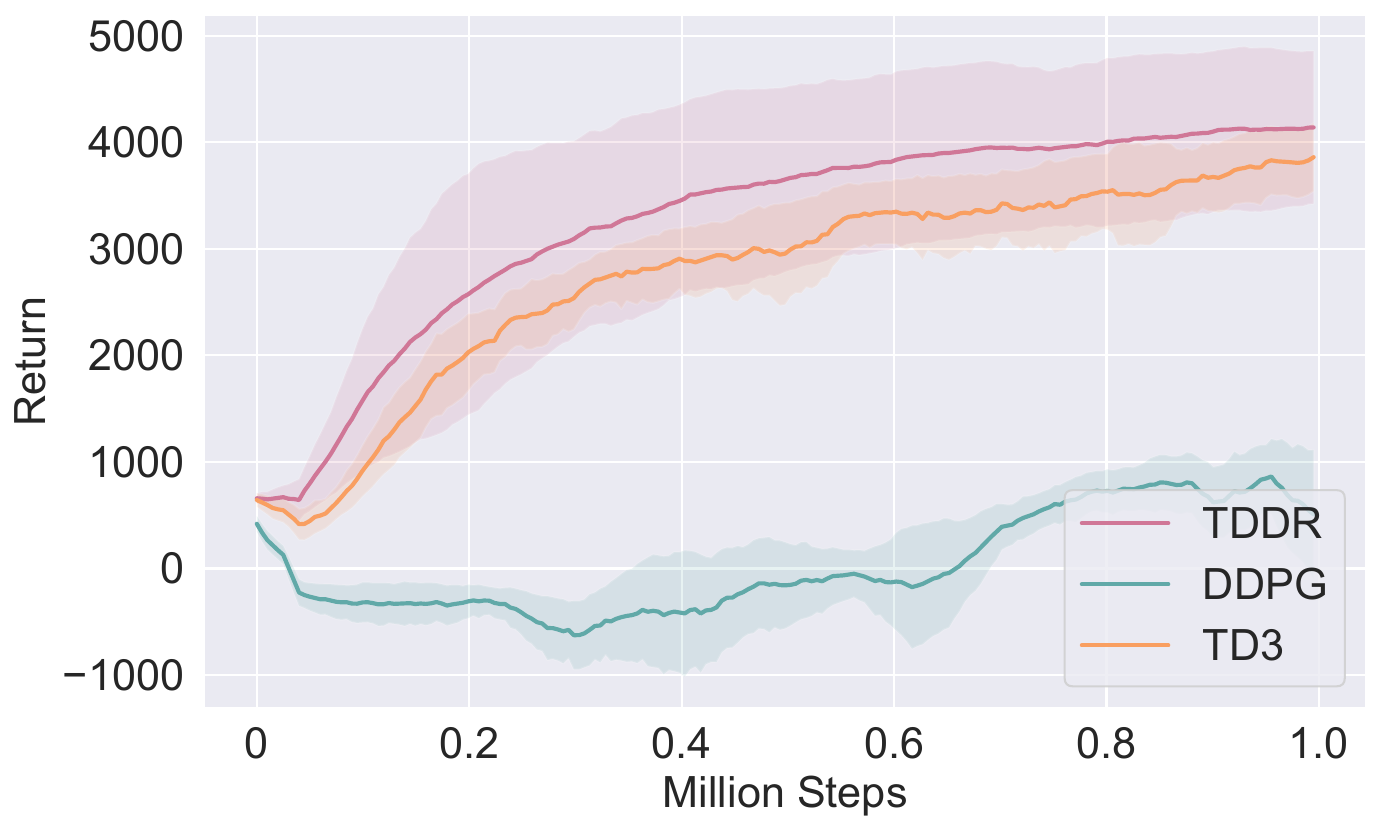}}
      \hfill
      \subcaptionbox{}{\includegraphics[width = 0.32\textwidth]{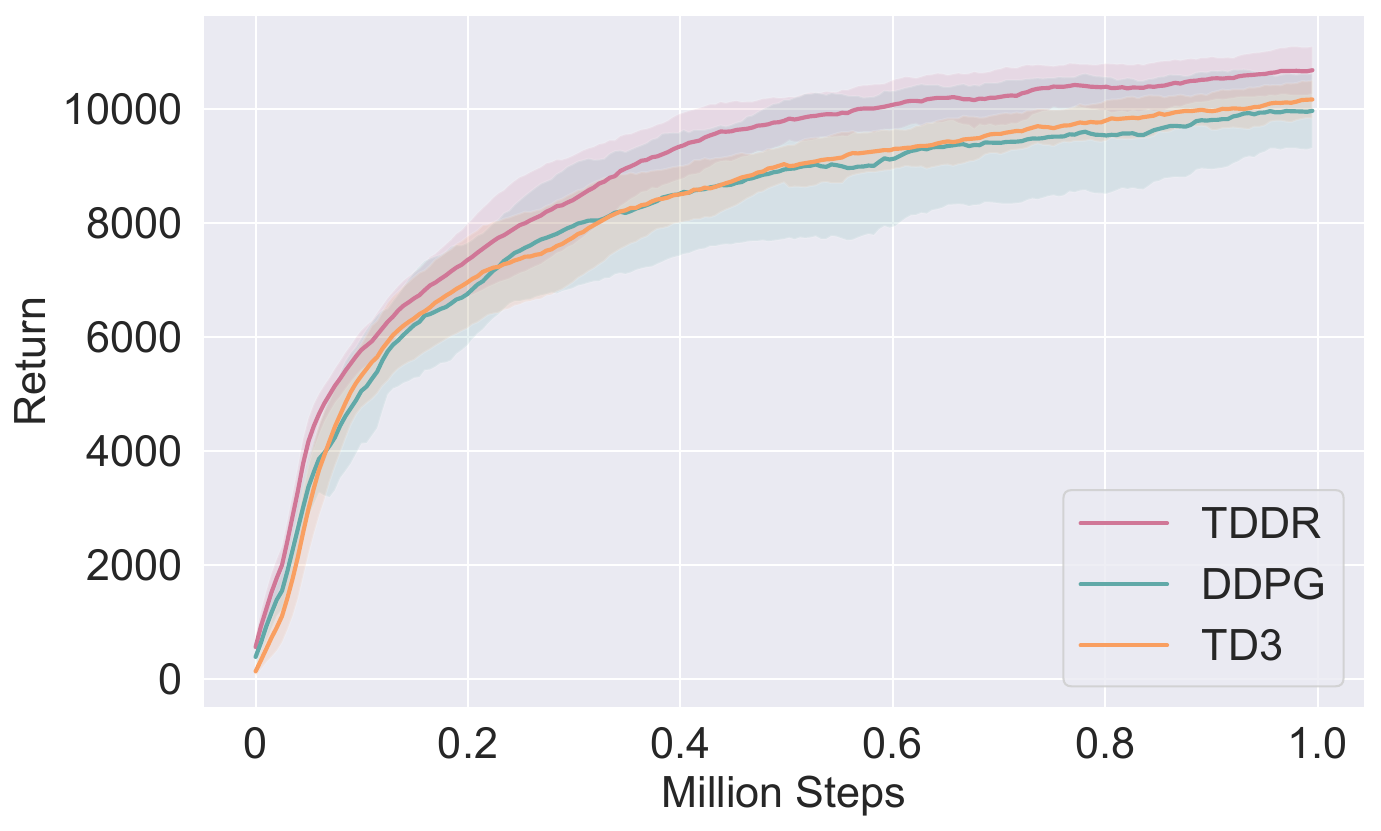}}
      \hfill
      \subcaptionbox{}{\includegraphics[width = 0.32\textwidth]{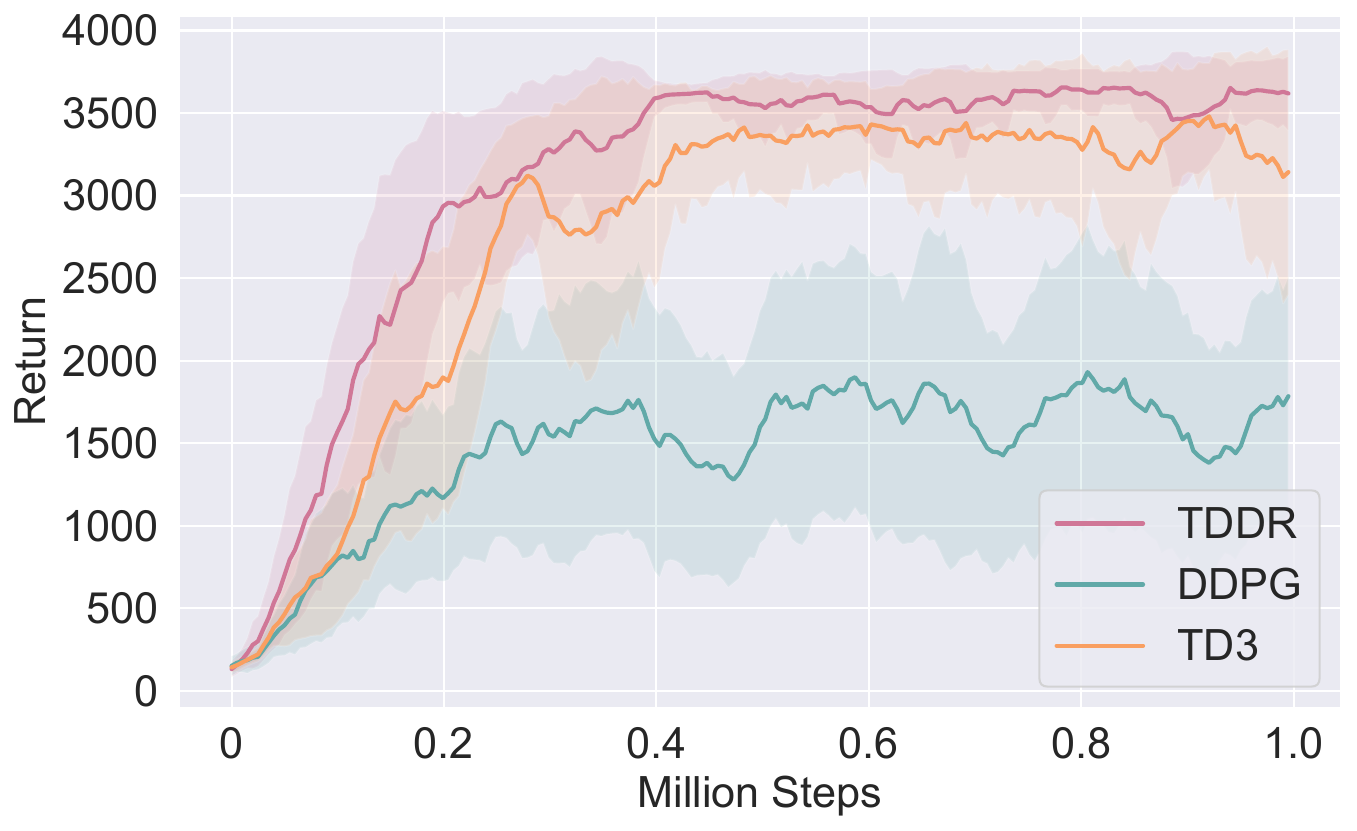}}
      \hfill
      \subcaptionbox{}{\includegraphics[width = 0.32\textwidth]{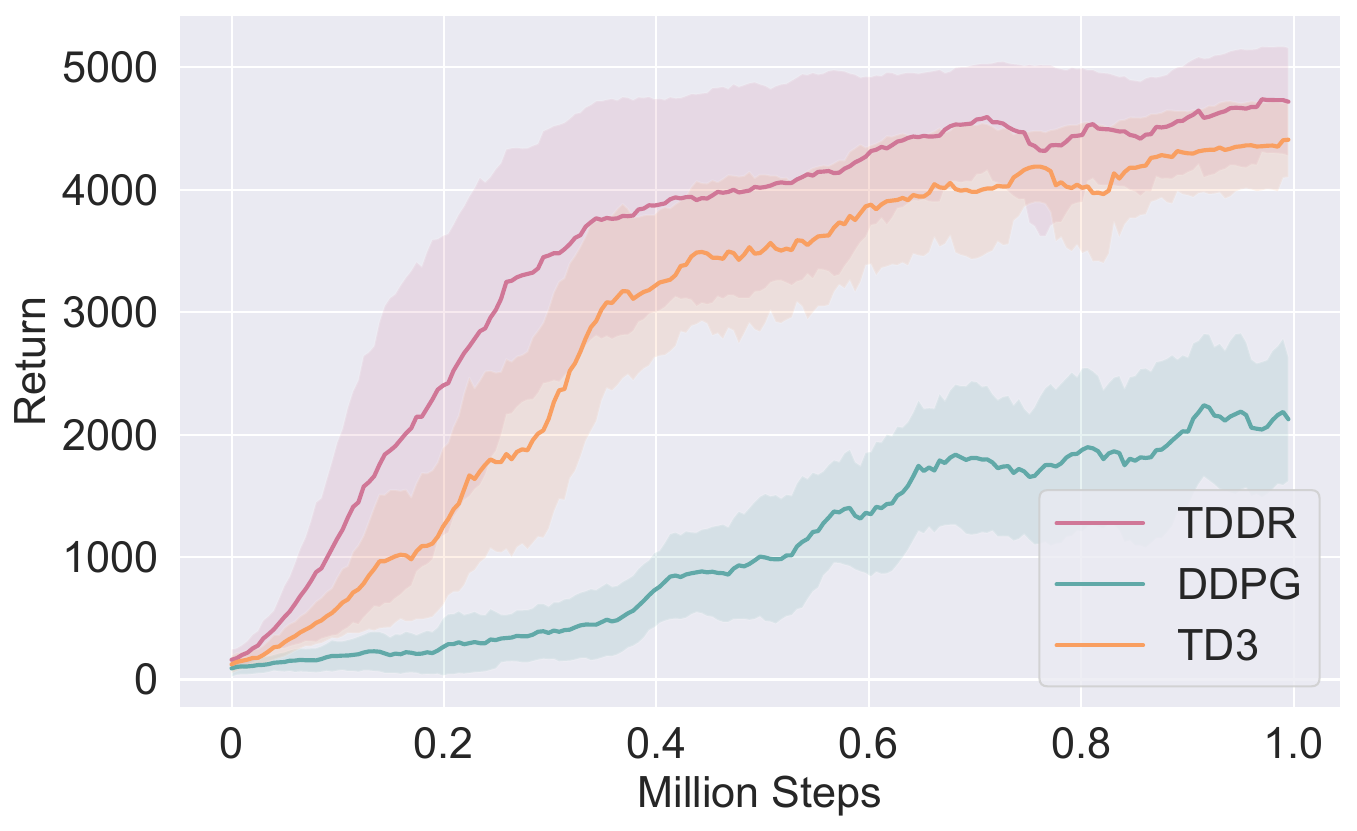}}
      \hfill
      \subcaptionbox{}{\includegraphics[width = 0.32\textwidth]{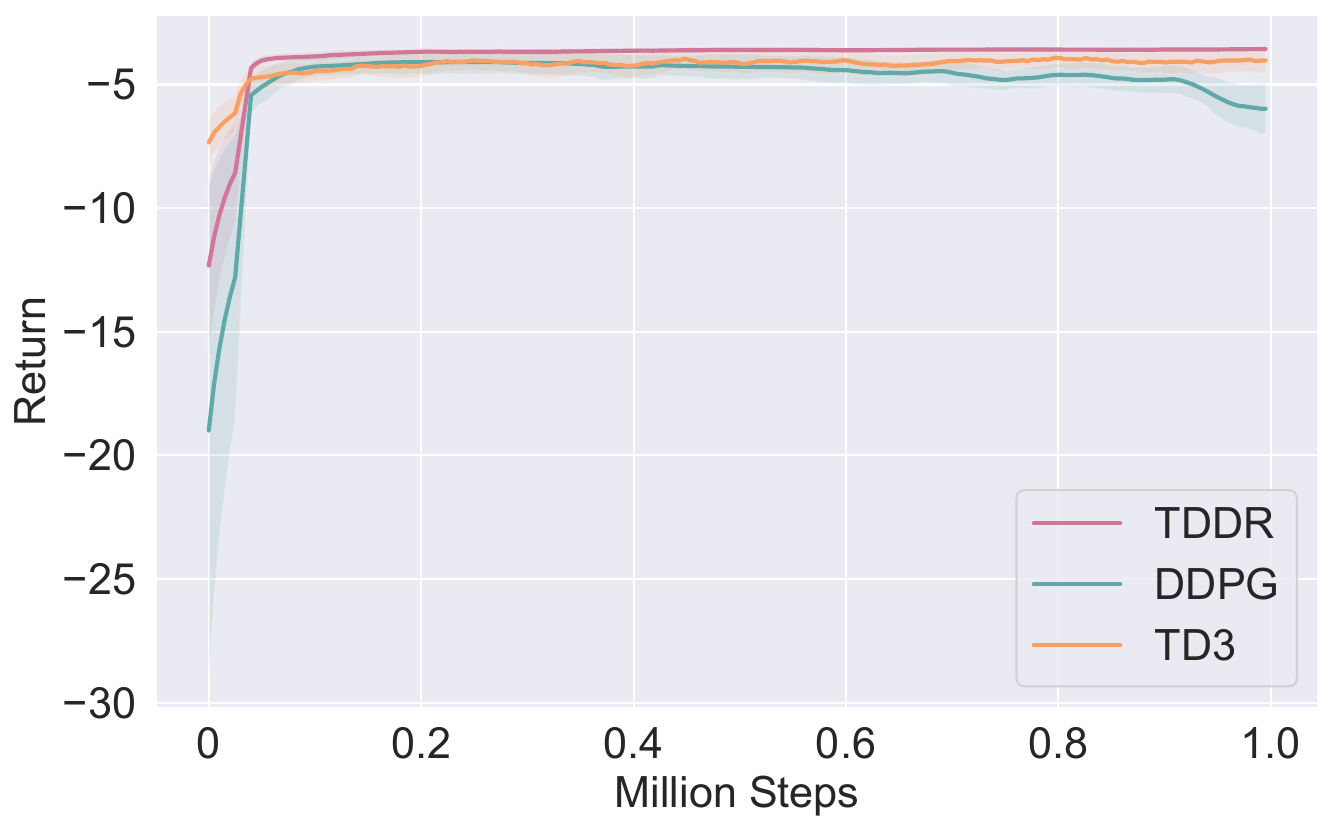}}
      \hfill
      \subcaptionbox{}{\includegraphics[width = 0.32\textwidth]{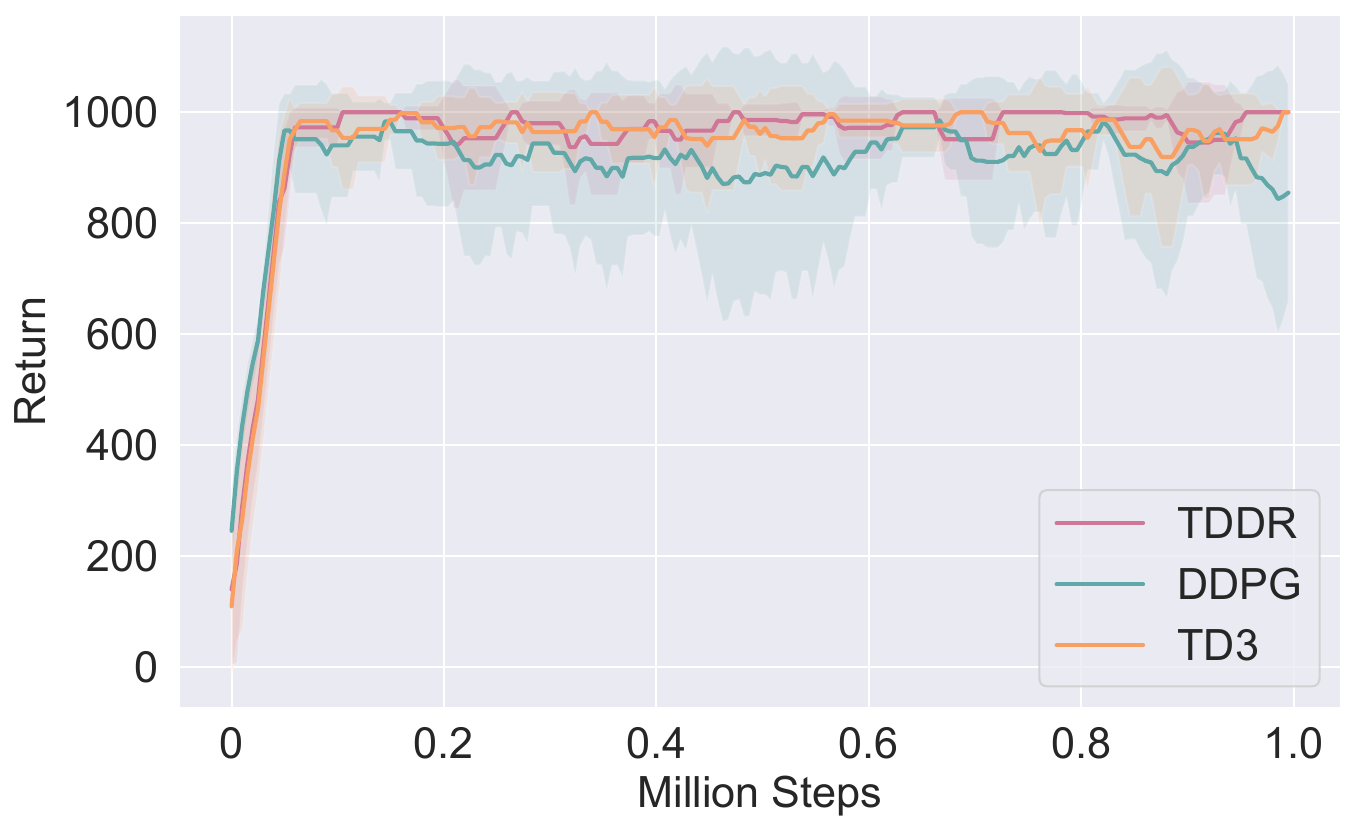}}
      \hfill
      \subcaptionbox{}{\includegraphics[width = 0.32\textwidth]{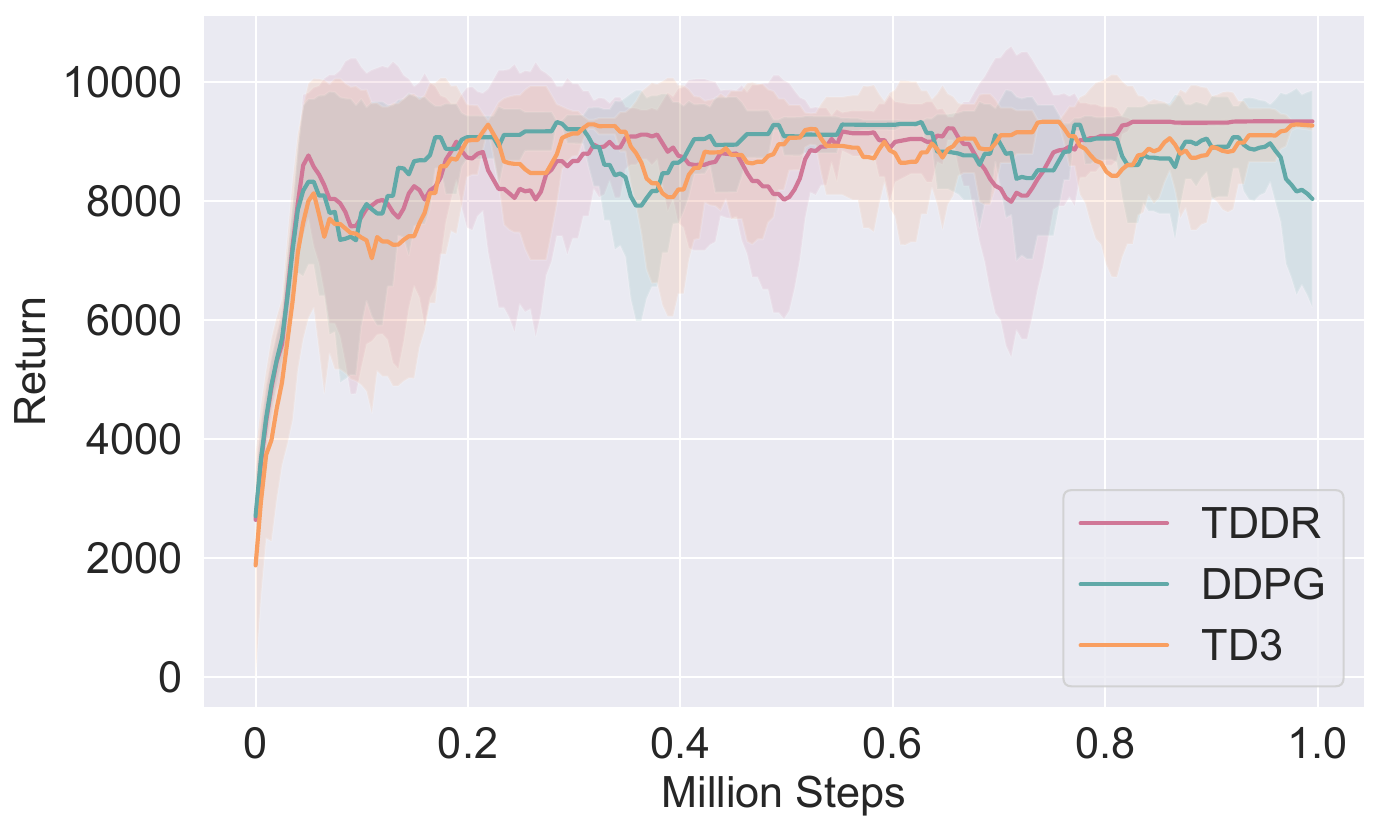}}
      \hfill
      \subcaptionbox{}{\includegraphics[width = 0.32\textwidth]{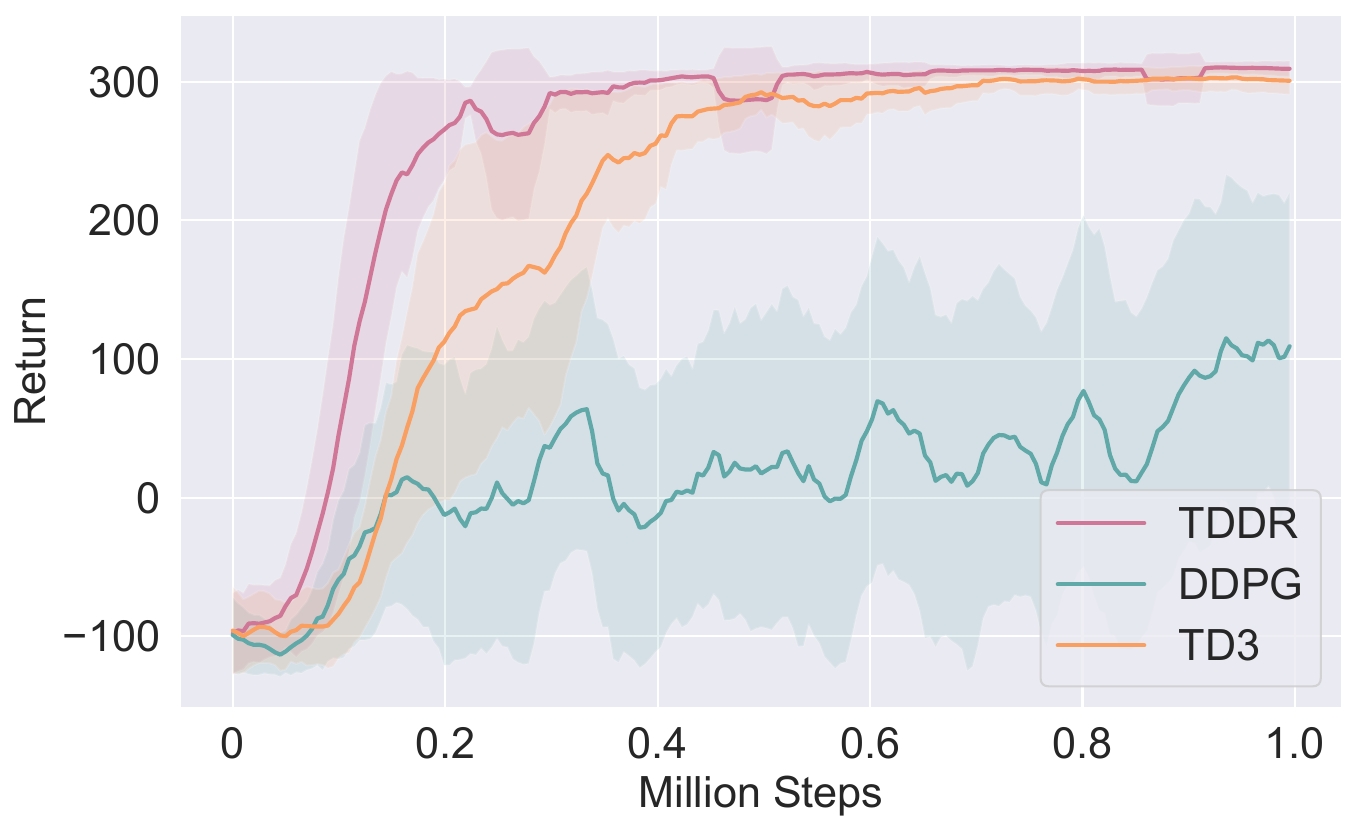}}
      \hfill
      \subcaptionbox{}{\includegraphics[width = 0.32\textwidth]{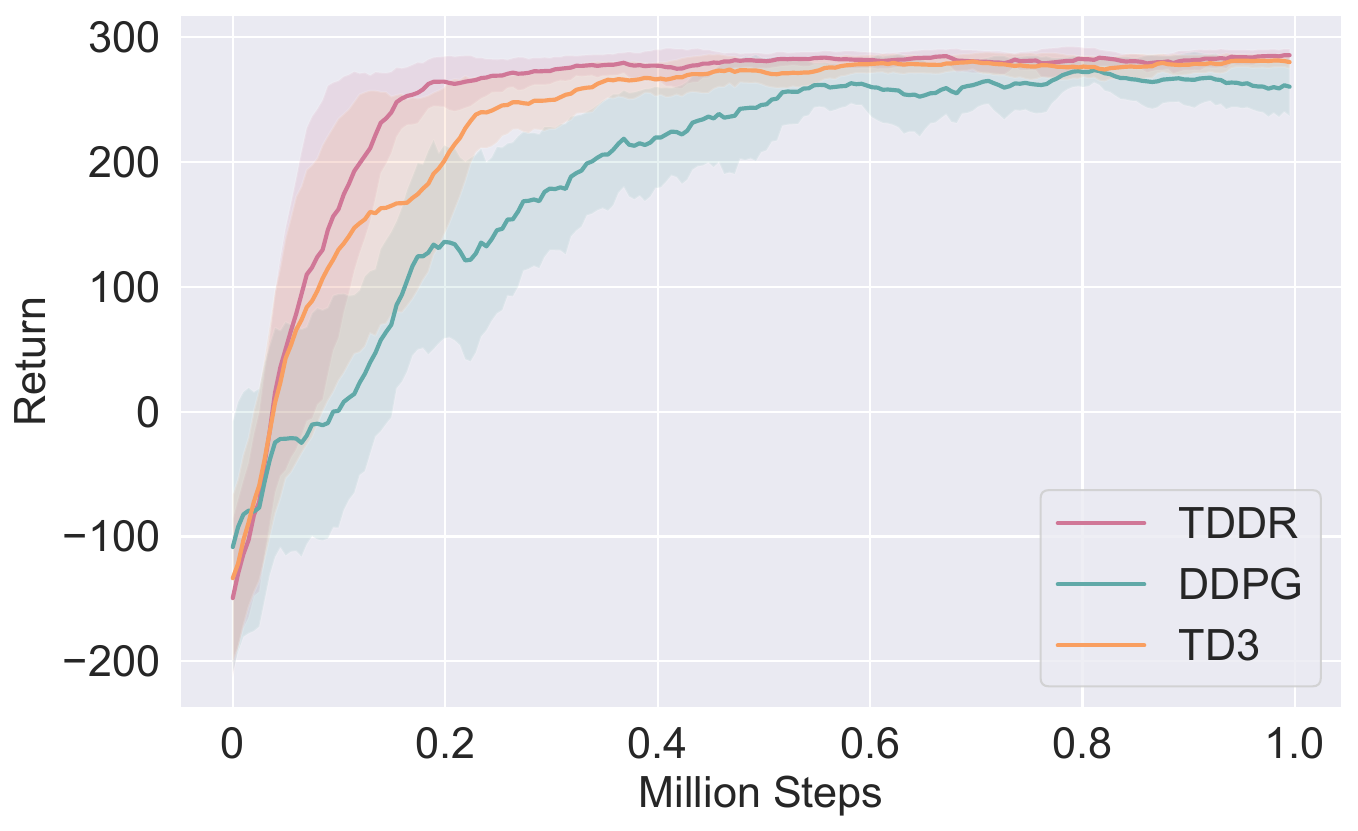}}
  \caption{Comparison of TDDR with DDPG and TD3 across nine environments. (a) Ant-v2, (b) HalfCheetah-v2, (c) Hopper-v2, (d) Walker2d-v2, (e) Reacher-v2, (f) InvertedPendulum-v2, (g) InvertedDoublePendulum-v2, (h) BipedalWalker-v3, (i) LunarLanderContinuous-v2.}
  \label{fig:baseline}
  \end{figure*}

\begin{table*}[!t] 
    \caption{Average return and standard deviation for Fig.~\ref{fig:baseline}}
      \centering
      \small
      \begin{tabular}{cccccc}
        \toprule
        Algorithms & TDDR & DDPG & TD3 \\
        \midrule
        Ant-v2 & $\bm{4128.64 \pm 726.16}$ & 640.67 $\pm$ 206.84 & 3811.08 $\pm$ 212.79  \\
        HalfCheetah-v2 & $\bm{10666.23 \pm 391.58}$ & 9958.33 $\pm$ 568.51 &  10104.80 $\pm$ 207.58 \\
        Hopper-v2 & $\bm{3631.20 \pm 160.41}$ & 1711.64 $\pm$ 167.48 &  3197.48 $\pm$ 281.71 \\
        Walker2d-v2 & $\bm{4733.41 \pm 423.20}$ & 2063.17 $\pm$ 230.83 &  4360.12 $\pm$ 228.91 \\
        Reacher-v2 & $\bm{-3.58 \pm 0.08}$ & -5.88 $\pm$ 0.73 & -4.01 $\pm$ 0.19 \\
        InvertedPendulum-v2 & $\bm{1000.00 \pm 0.00}$ & 869.17 $\pm$ 60.03 & 968.21 $\pm$ 24.81 \\
        InvertedDoublePendulum-v2 & $\bm{9342.17 \pm 12.13}$ & 8278.83 $\pm$ 551.77 & 9276.81 $\pm$ 37.19 \\
        BipedalWalker-v3 & $\bm{309.71 \pm 4.45}$ & 113.34 $\pm$ 37.99 & 301.14 $\pm$ 6.99 \\
        LunarLanderContinuous-v2 & $\bm{285.08 \pm 1.53}$ & 258.91 $\pm$ 8.11 & 281.08 $\pm$ 2.54 \\
       \bottomrule
      \end{tabular}
        \label{table:baseline}
    \end{table*}

\subsection{Comparison with DARC, SD3, and GD3}

We have already discussed the additional hyperparameters in DARC, SD3, and GD3, and their effects on these algorithms. Two sets of hyperparameters, designated as ``better'' and ``worse'' are listed in Table~\ref{table.hyperpara}. The comparison of TDDR with these three algorithms using each set of hyperparameters is discussed below. 

Note that in GD3, the bias term is set to $b=2$ in all environments, while the activation functions are environment-dependent. Specifically, second-order polynomial functions are used for BipedalWalker, Hopper, and Walker2d, while third-order polynomial functions are used for the other six environments.

\subsubsection{DARC, SD3, and GD3 with better hyperparameters}

In this comparison, we evaluate TDDR against DARC, SD3, and GD3 using the hyperparameters that yield better performance. The overall performance comparison is illustrated in Fig.~\ref{fig:better}, and the numerical results are detailed in Table~\ref{table.better}. The following observations can be made from these results.

1. DARC achieves the best average return and smallest standard deviation in two environments: Ant and HalfCheetah. DARC has the optimal average return in Walker2d.

2. TDDR performs similarly to SD3 in HalfCheetah. In Reacher and InvertedDoublePendulum, TDDR has the optimal standard deviation and average return. TDDR has the optimal average return in Hopper.

4. GD3 and SD3 achieve higher average returns than TDDR in Ant. In Walker2d, GD3 and SD3 have better stability compare with TDDR.

5. TDDR and SD3 exhibit relatively unstable performance in BipedalWalker. While TDDR achieves the highest average return in BipedalWalker, its larger standard deviation suggests relatively unstable performance. In contrast, DARC and GD3 demonstrate better stability in this environment.

6. In InvertedPendulum, TDDR and DARC exhibit good performance, while GD3 and SD3 outperform TD3 and DDPG. In LunarLanderContinuous, SD3 achieves the best average return, with TDDR surpassing DARC and GD3 in terms of average return, but GD3 achieves the smallest standard deviation.

In conclusion, different algorithms show various strengths and weaknesses in different environments, and the most suitable algorithm needs to be selected according to the environment.

\begin{figure*}
  \centering
      \subcaptionbox{}{\includegraphics[width = 0.32\textwidth]{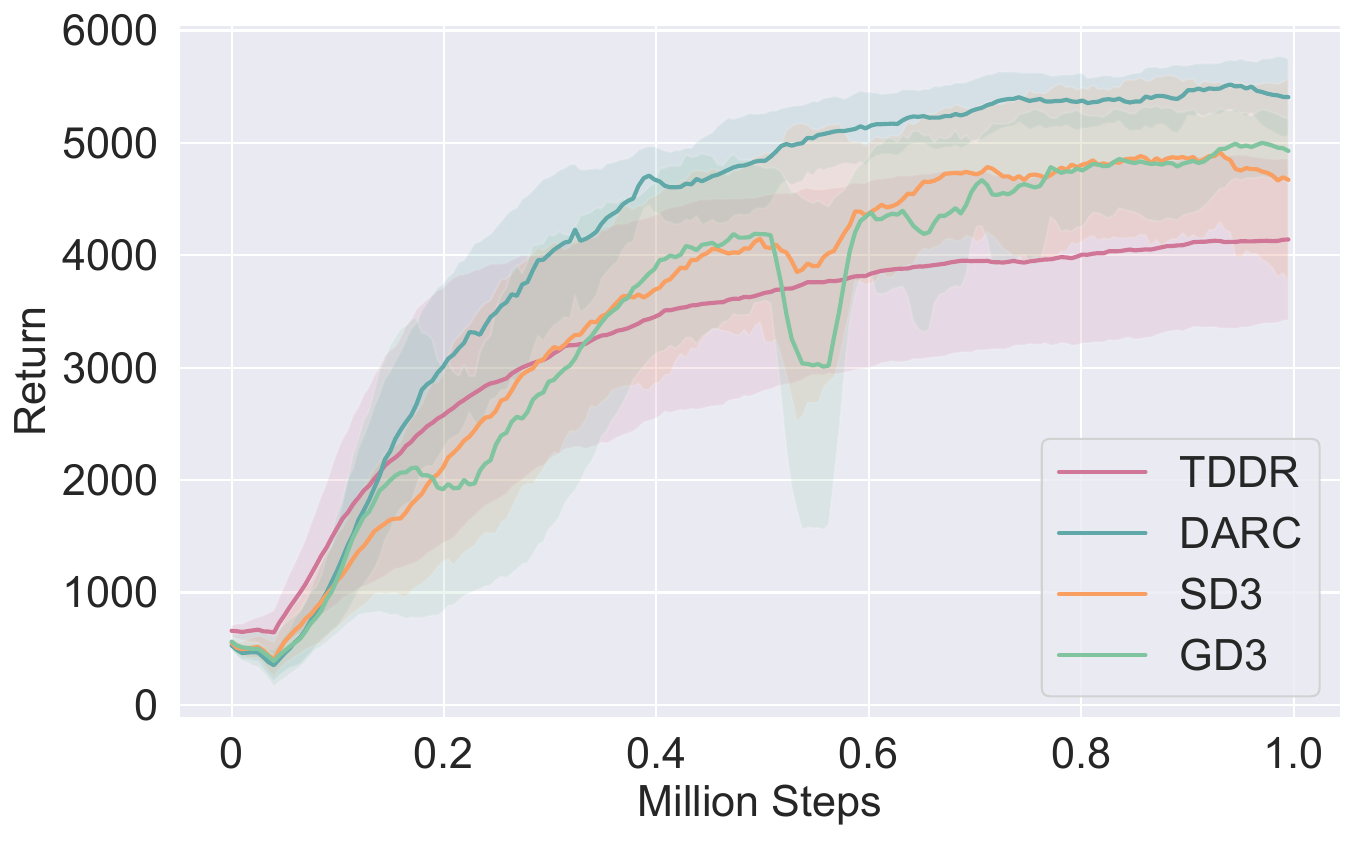}}
      \hfill
      \subcaptionbox{}{\includegraphics[width = 0.32\textwidth]{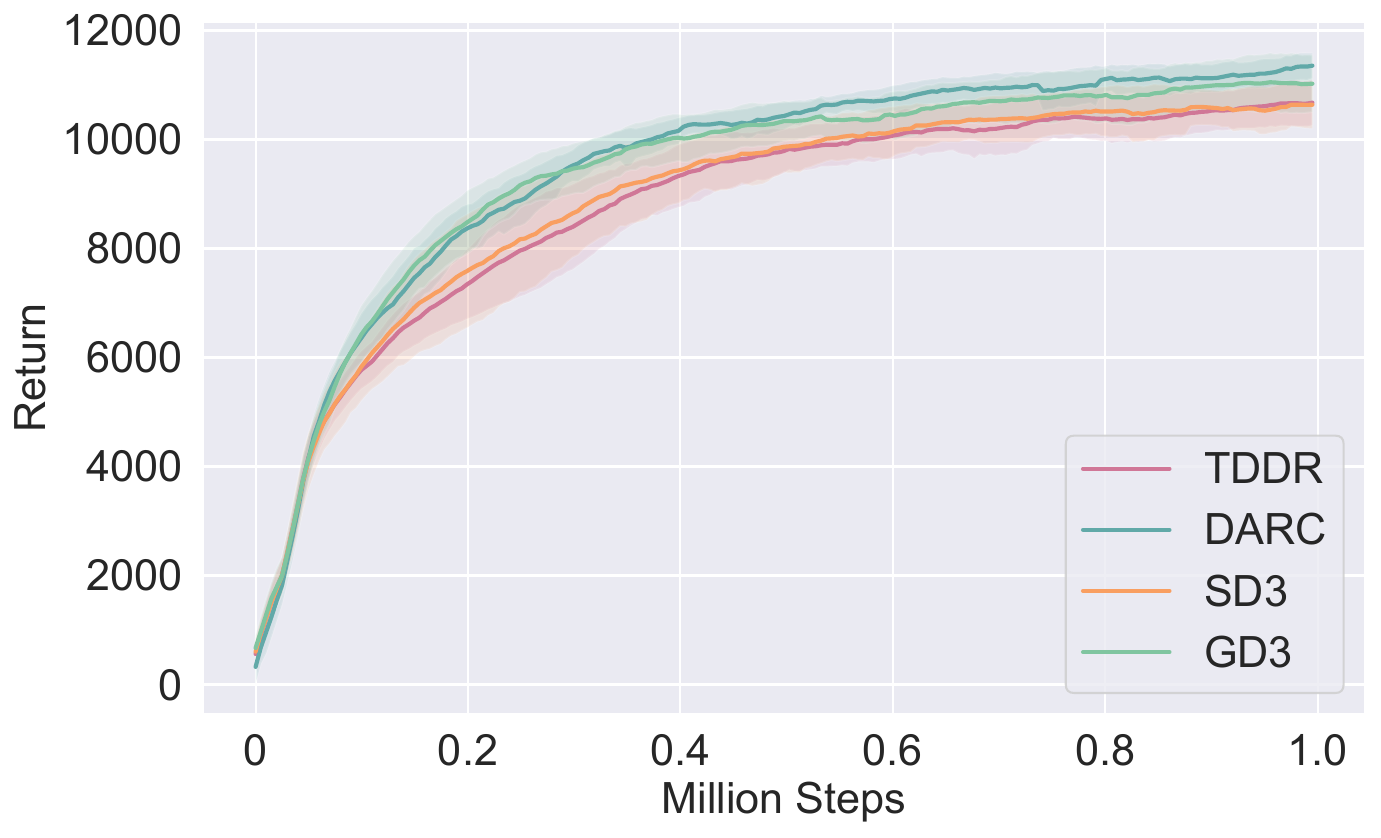}}
      \hfill
      \subcaptionbox{}{\includegraphics[width = 0.32\textwidth]{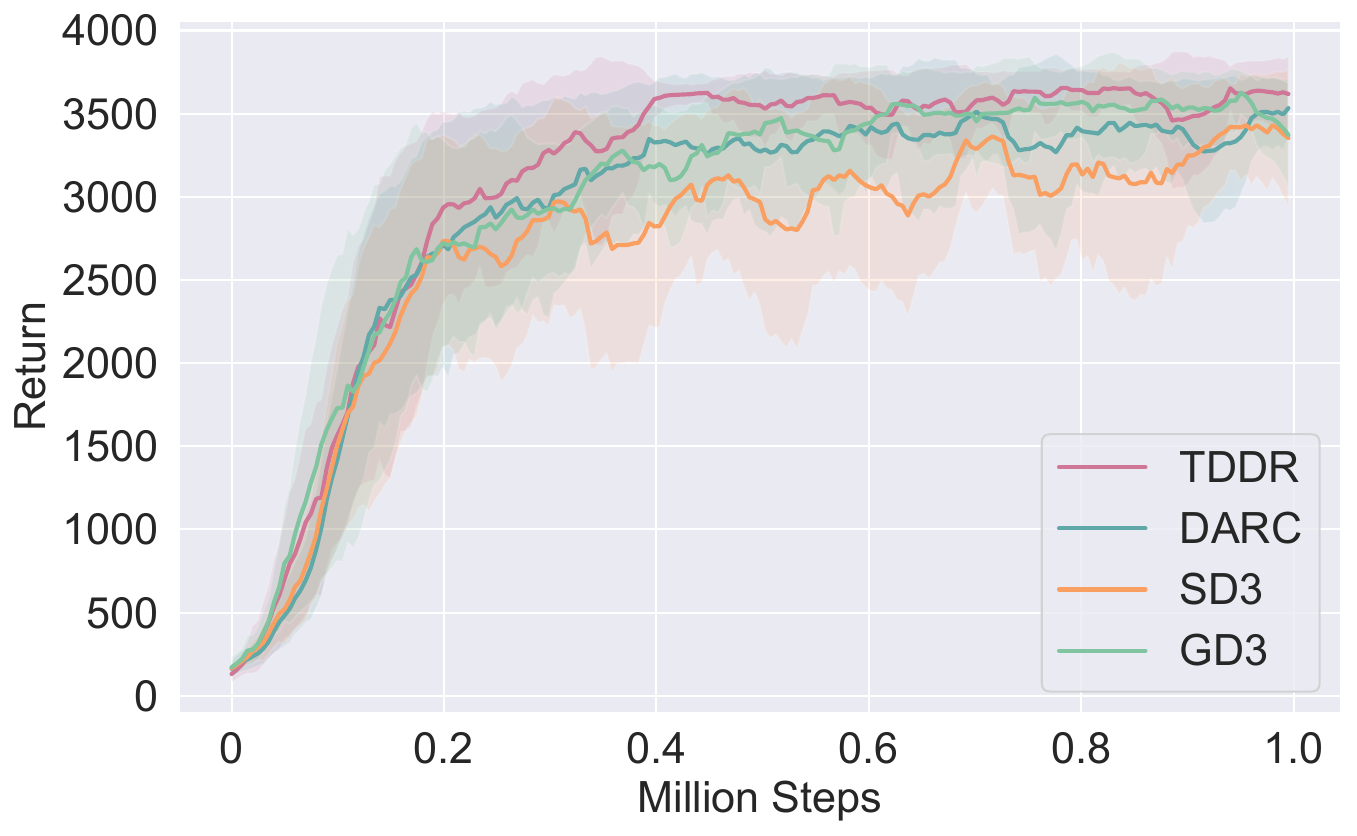}}
      \hfill
      \subcaptionbox{}{\includegraphics[width = 0.32\textwidth]{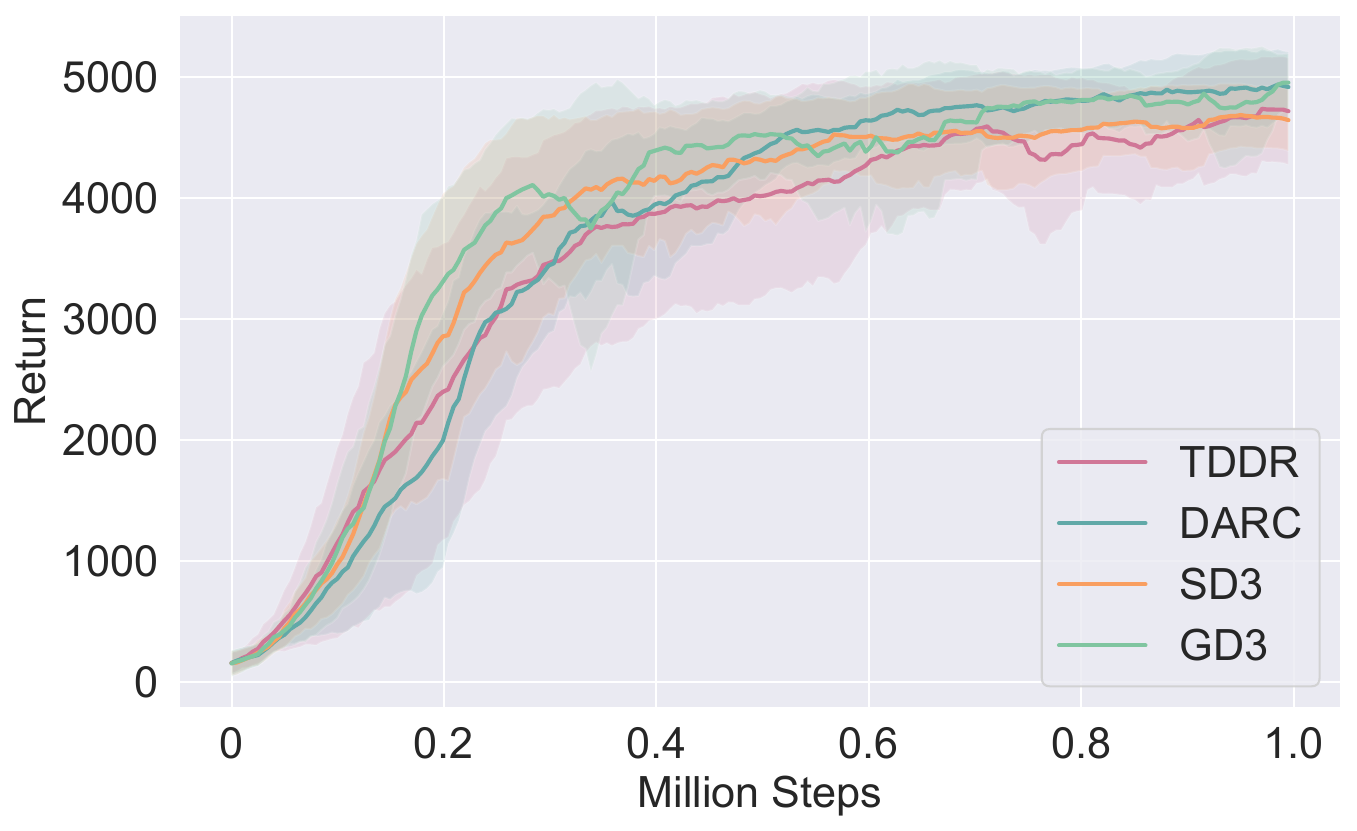}}
      \hfill
      \subcaptionbox{}{\includegraphics[width = 0.32\textwidth]{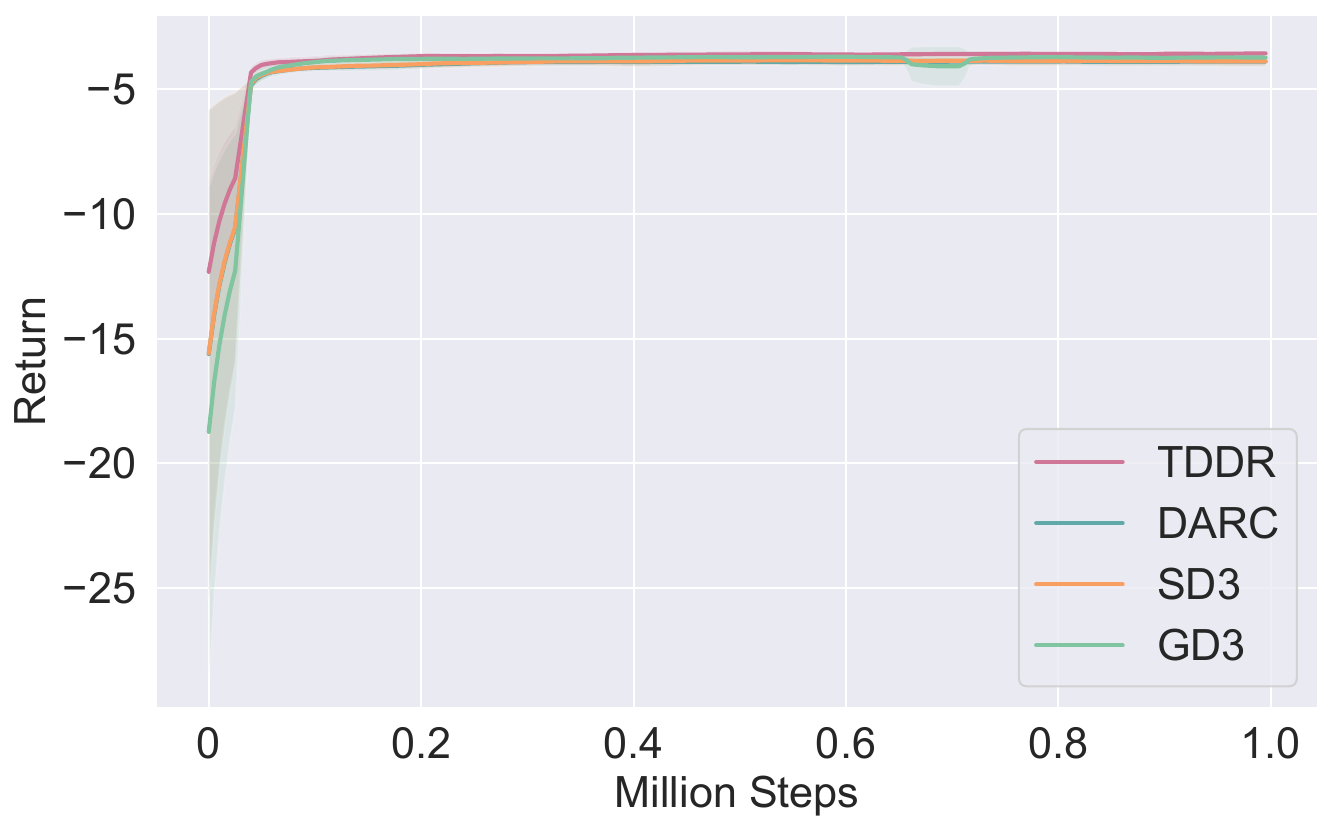}}
      \hfill
      \subcaptionbox{}{\includegraphics[width = 0.32\textwidth]{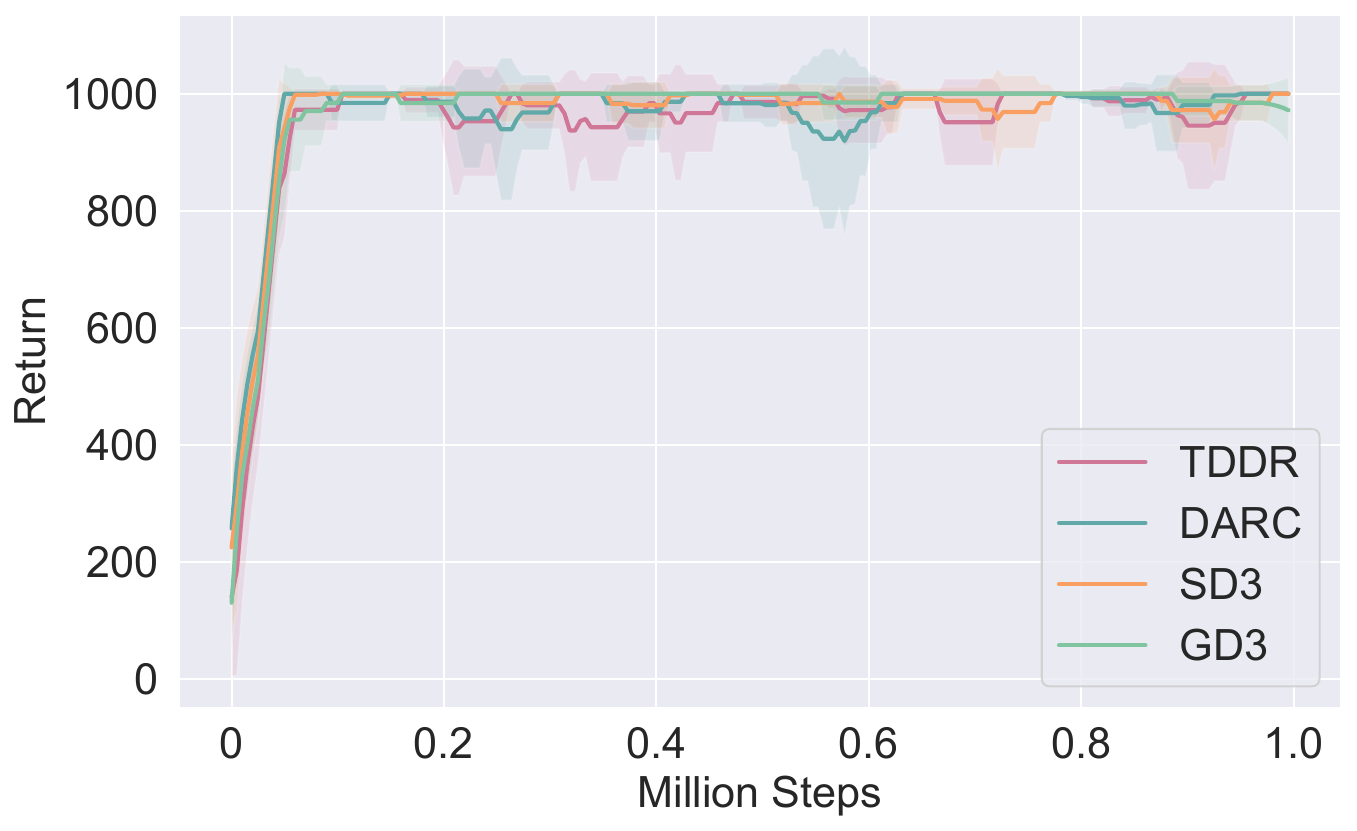}}
      \hfill
      \subcaptionbox{}{\includegraphics[width = 0.32\textwidth]{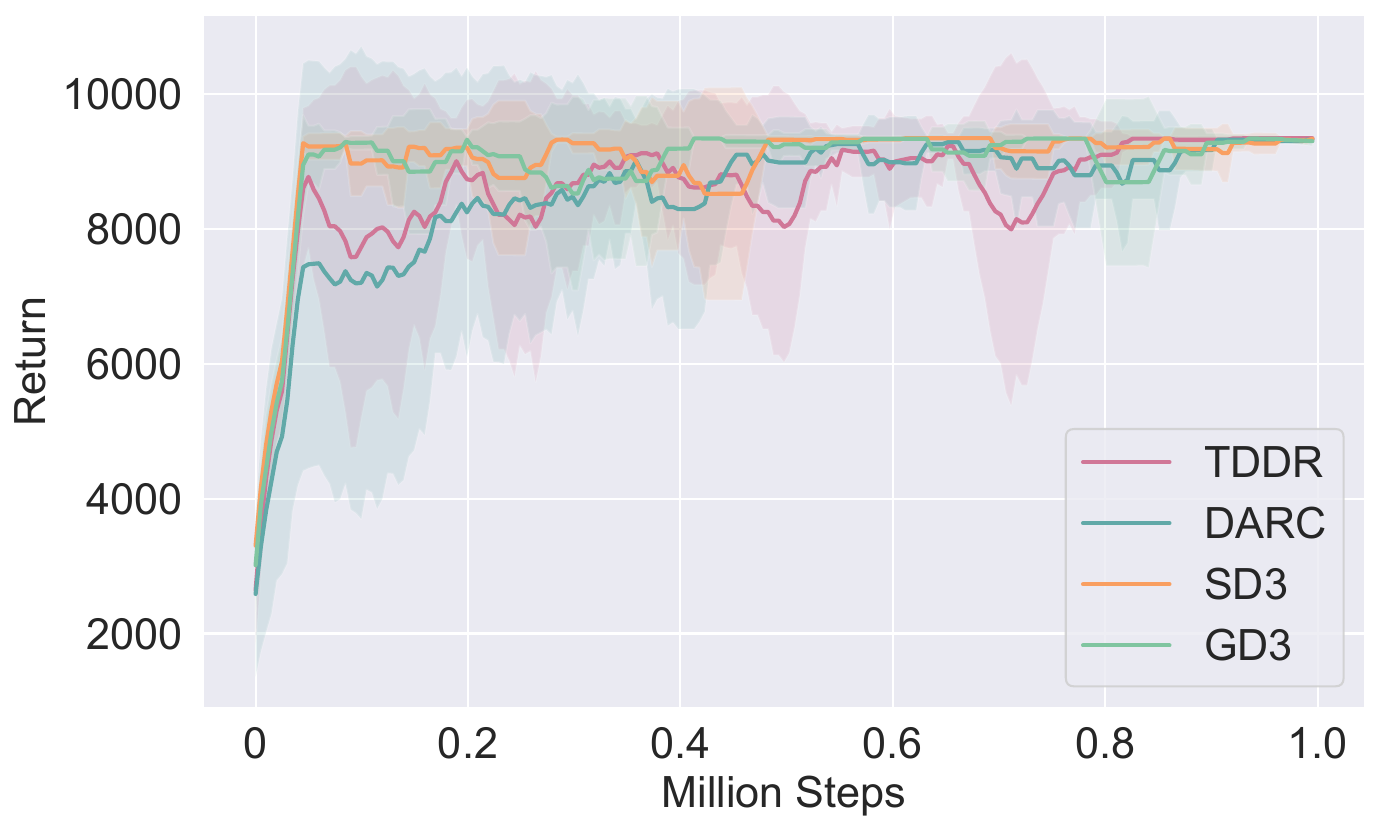}}
      \hfill
      \subcaptionbox{}{\includegraphics[width = 0.32\textwidth]{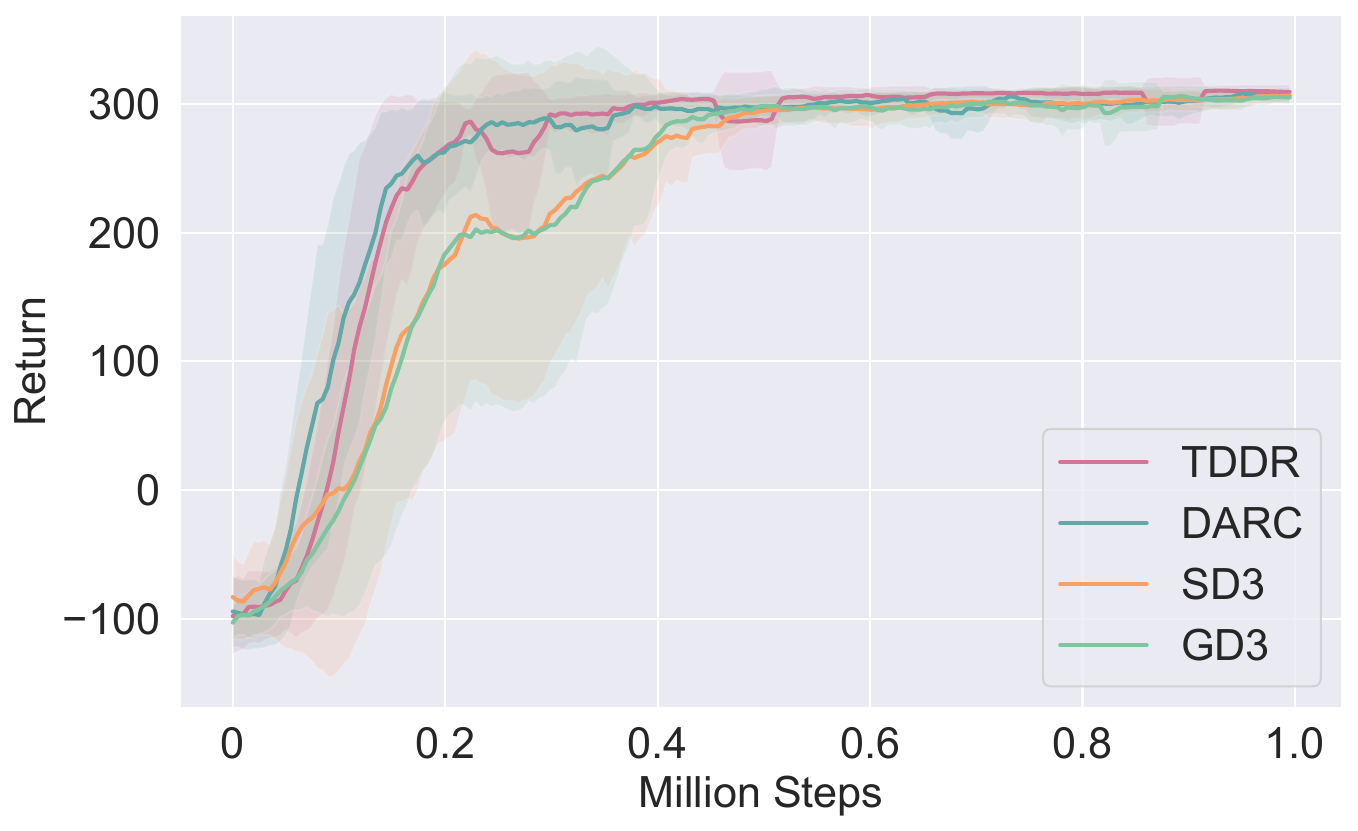}}
      \hfill
      \subcaptionbox{}{\includegraphics[width = 0.32\textwidth]{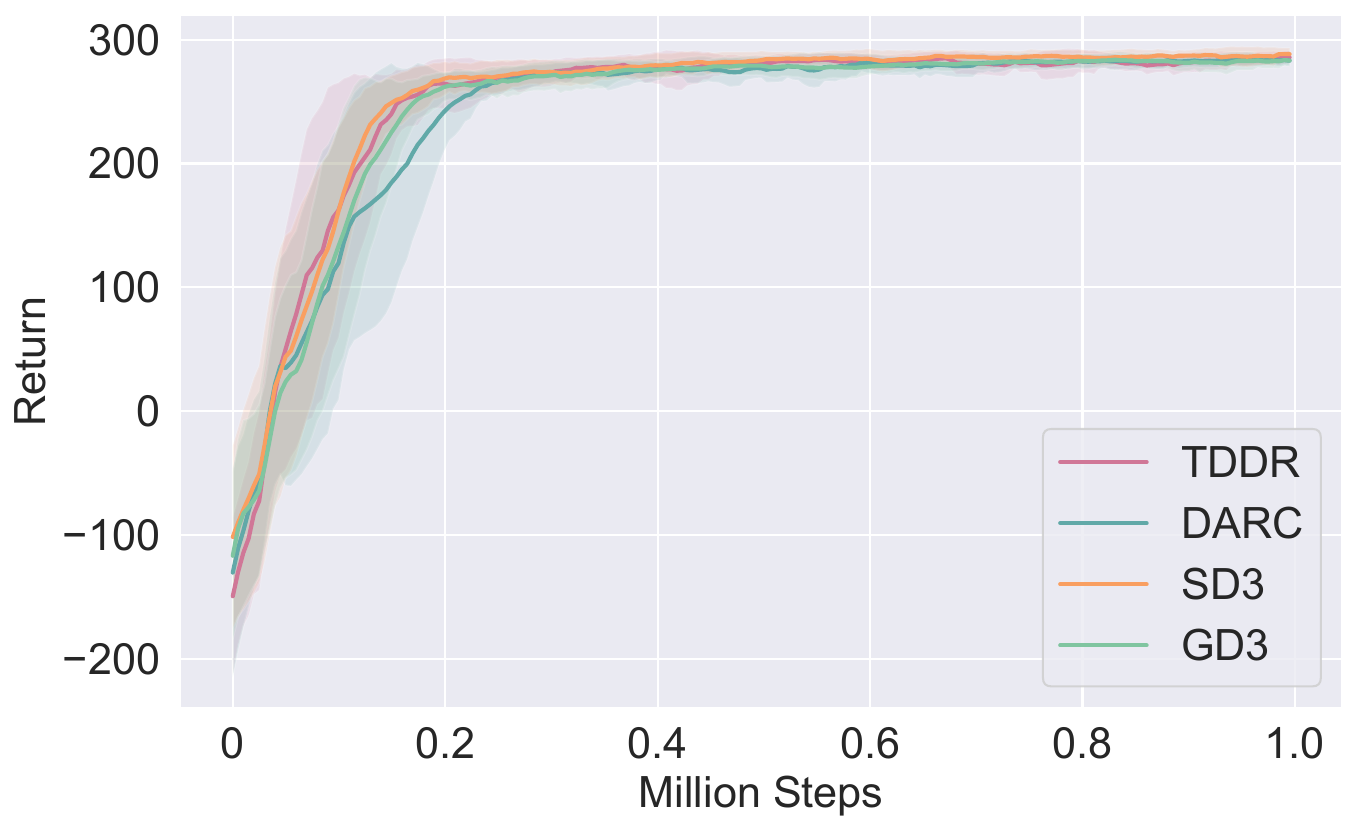}}
\caption{Comparison of TDDR with DARC, SD3, and GD3 with better hyperparameters across nine environments.  (a) Ant-v2, (b) HalfCheetah-v2, (c) Hopper-v2, (d) Walker2d-v2, (e) Reacher-v2, (f) InvertedPendulum-v2, (g) InvertedDoublePendulum-v2, (h) BipedalWalker-v3, (i) LunarLanderContinuous-v2.}
  \label{fig:better}
  \end{figure*}
\begin{table*}[!h] 
    \caption{Average return and standard deviation for Fig.~\ref{fig:better} }
    \centering
    \small
    \begin{tabular}{cccccc}
      \toprule
      Algorithms & TDDR & DARC & SD3 & GD3 \\
      \midrule 
      Ant-v2 & 4128.64 $\pm$ 726.16 & $\bm{5437.90 \pm 214.99}$ & 4733.87 $\pm$ 604.99 & 4989.52 $\pm$ 242.39 \\
      HalfCheetah-v2 & 10666.23 $\pm$ 391.58 & $\bm{11296.50 \pm 203.87}$ & 10637.48 $\pm$ 280.24 & 11031.75 $\pm$ 522.95 \\
      Hopper-v2 & $\bm{3631.20 \pm 160.41}$ & 3509.98 $\pm$ 114.13 & 3386.04 $\pm$ 162.18 & 3473.73 $\pm$ 113.60  \\
      Walker2d-v2 & 4733.41 $\pm$ 423.20 & $\bm{4900.95 \pm 243.97}$ & 4672.26 $\pm$ 241.39 & 4855.87 $\pm$ 349.54 \\
      Reacher-v2 & $\bm{-3.58 \pm 0.08}$ & -3.90 $\pm$ 0.19 & -3.88 $\pm$ 0.19 & -3.73 $\pm$ 0.18  \\
      InvertedPendulum-v2 & $\bm{1000.00 \pm 0.00}$ & $\bm{1000.00 \pm 0.00}$ & 983.07 $\pm$ 33.85 & 983.11 $\pm$ 33.77  \\
      InvertedDoublePendulum-v2 & $\bm{9342.17 \pm 12.13}$ & 9302.25 $\pm$ 31.39 & 9317.47 $\pm$ 33.72 & 9314.28 $\pm$ 32.61  \\
      BipedalWalker-v3 & $\bm{309.71 \pm 4.45}$ & 306.76 $\pm$ 2.76 & 306.00 $\pm$ 4.96 & 304.61 $\pm$ 2.41  \\
      LunarLanderContinuous-v2 & 285.08 $\pm$ 1.53 & 283.63 $\pm$ 1.40 & $\bm{286.69 \pm 2.03}$ & 282.34 $\pm$ 1.28  \\
     \bottomrule
    \end{tabular}
    \label{table.better}
  \end{table*}

\subsubsection{DARC, SD3, and GD3 with worse hyperparameters} 

In a separate comparison, we evaluate TDDR against DARC, SD3, and GD3 using the hyperparameters that yield worse performance. Only the Ant, HalfCheetah, and Walker2d environments are tested here, as their performance with better hyperparameters surpasses or is close to TDDR in terms of average return and standard deviation. The overall performance comparison is illustrated in Fig.\ref{fig:worse}, and the numerical results are detailed in Table\ref{table.worse}. The following observations can be made from these results.

1. In all three environments, the standard deviation of DARC significantly increases. Similarly, GD3 shows a substantial increase in standard deviation across all environments except HalfCheetah. Except for Ant, the standard deviation of SD3 also significantly increases.

2. We illustrate using DARC as an example, when the hyperparameters are appropriately chosen, the average return in Ant is 5437.90 with a standard deviation of 214.99. However, when the hyperparameters are not appropriately chosen, the average return decreases to 3483.91, while the standard deviation significantly increases to 1842.89. This indicates that under worse hyperparameter selection, the performance of DARC in Ant is degraded. In the remaining two environments, the performance of DARC is also significantly affected by hyperparameter selection. Additionally, GD3 and SD3 similarly show clear signs of being affected by hyperparameter choices.

In conclusion, the quality of hyperparameters significantly affects the performance and stability of RL algorithms. With better hyperparameters, algorithms usually achieve higher average returns and smaller standard deviations. However, when hyperparameters are worse, the performance and stability of the algorithms decrease significantly. Therefore, carefully tuning hyperparameters to optimize the performance of RL algorithms is crucial. In addition, compared to DARC, SD3, and GD3, which require additional hyperparameters, TDDR, requiring none, exhibits better stability. It is worth noting that the additional hyperparameters are relative to TD3. For instance, DARC includes a regularization parameter, while SD3 and GD3 require a hyperparameter such as NNS.

\begin{figure*}[!t]

      \subcaptionbox{}{\includegraphics[width = 0.32\textwidth]{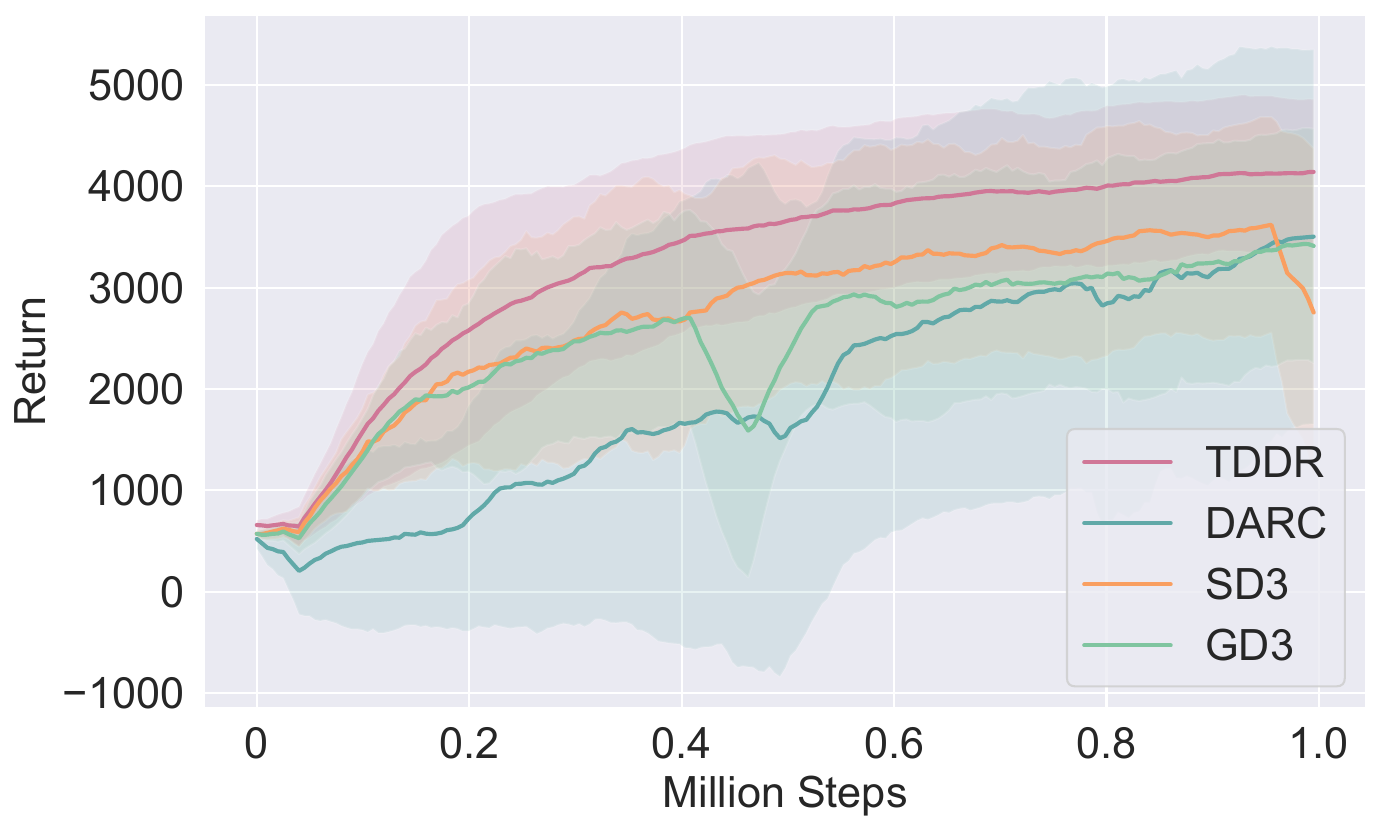}}
      \hfill
      \subcaptionbox{}{\includegraphics[width = 0.32\textwidth]{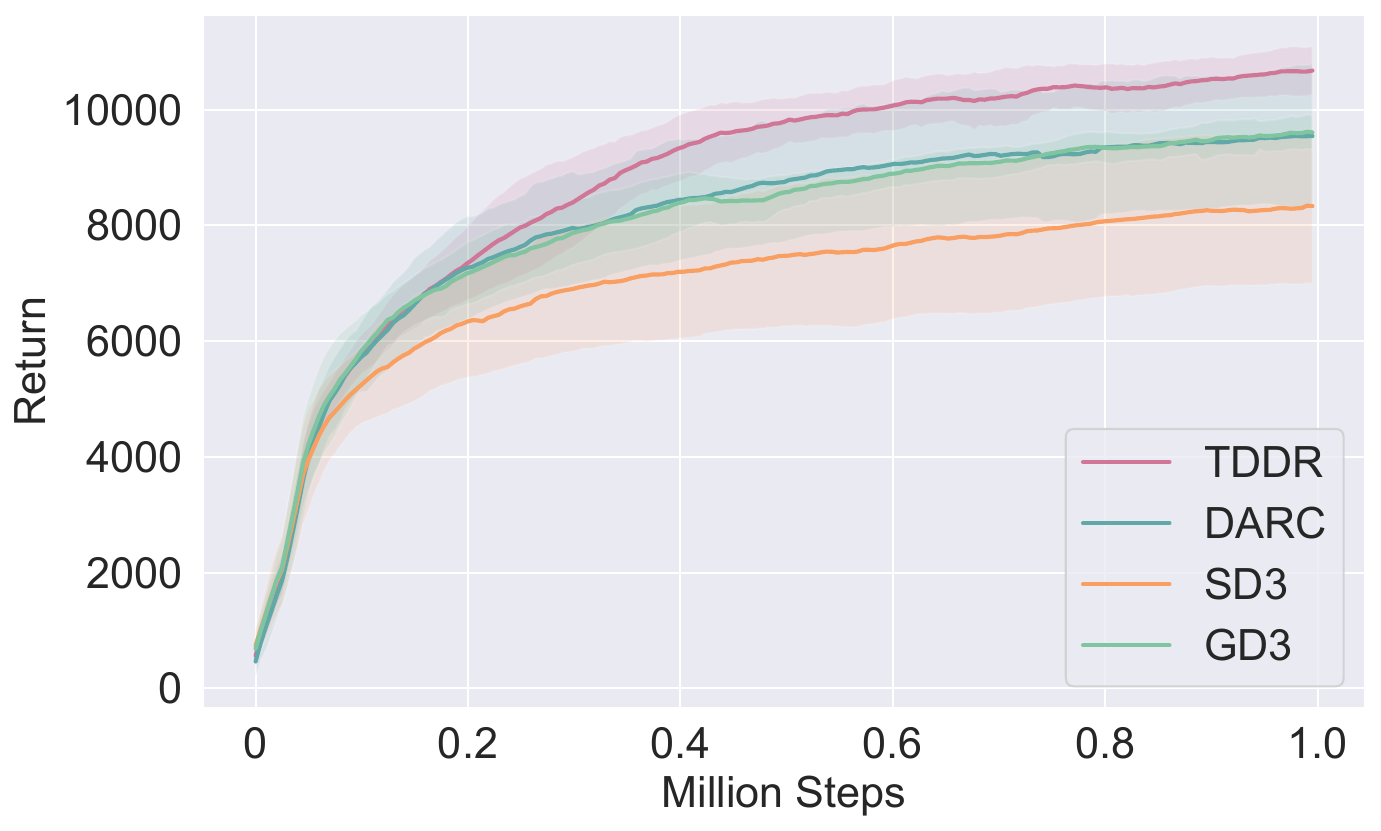}}
      \hfill
      \subcaptionbox{}{\includegraphics[width = 0.32\textwidth]{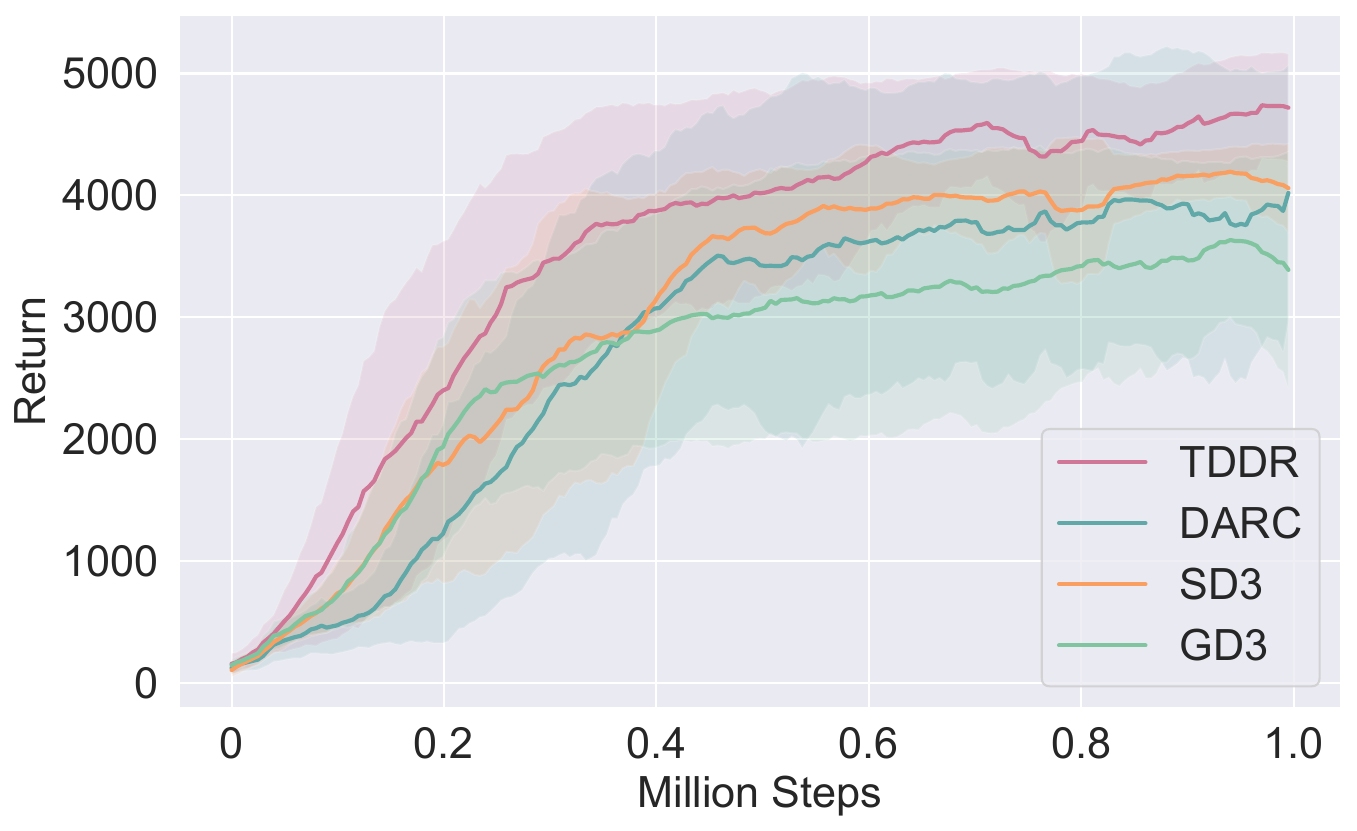}}
  \caption{Comparison of TDDR with DARC, SD3, and GD3 with worse hyperparameters across three environments. (a) Ant-v2, (b) HalfCheetah-v2, (c) Walker2d-v2.}
  \label{fig:worse}
  \end{figure*}

\begin{table*}[!t] 
    \caption{Average return and standard deviation for Fig.~\ref{fig:worse}}
    \centering
    \small
    \begin{tabular}{cccccc}
      \toprule
      Algorithms & TDDR & DARC & SD3 & GD3 \\
      \midrule 
      Ant-v2 & $\bm{4128.64 \pm 726.16}$ & 3483.91 $\pm$ 1842.89 & 3092.57 $\pm$ 159.92 & 3414.89 $\pm$ 1119.75 \\
      HalfCheetah-v2 & $\bm{10666.23 \pm 391.58}$ & 9536.55 $\pm$ 1201.46 & 8288.06 $\pm$ 1291.81 & 9601.94 $\pm$ 214.92 \\
      Walker2d-v2 & $\bm{4733.41 \pm 423.20}$ & 3925.50 $\pm$ 949.91 & 4124.68 $\pm$ 236.16 & 3518.46 $\pm$ 737.92  \\
     \bottomrule
    \end{tabular}
    \label{table.worse}
  \end{table*}

\subsection{Ablation Discussion}

TDDR consists of two major feature components: DA-CDQ and CRA, which are are analyzed here using ablation experiments.

First, the DA-CDQ generates actions as in \eqref{aprime}. When one actor is removed,  either $\min_{i=1,2}(Q_{\theta_i^{\prime}}(s^{\prime},a_1^{\prime}))$ or $\min_{i=1,2}(Q_{\theta_i^{\prime}}(s^{\prime},a_2^{\prime}))$ would disappear, rendering TDDR invalid, hence making it impossible to conduct this part of the ablation study.

At the same time, it can be observed that when removing an actor, CRA becomes ineffective, meaning that $\delta_1$ and $\delta_2$ parts cease to be effective. At this point, TDDR becomes an algorithm based on a single actor, which is equivalent to TD3. Assuming that $\pi_{\phi_2}$ is deleted, it can be found that CRA of TDDR contains only $\delta_1$, which can no longer be compared with $\delta_2$, thus losing its regularization significance. Hence, the performance evaluation of TDDR compared to TD3 can be regarded as an ablation study with one actor removed.

Second, we analyze CRA. Compared to removing an actor, removing either $\delta_1$ or $\delta_2$ directly destroys TDDR. First, if it is impossible to compare $\delta_1$ and $\delta_2$, then it is not feasible to calculate the TD target, thus \eqref{psiTDDR} are invalid. Second, it would lead to the inability to update the hyperparameters of the critic target network because in TDDR one actor is responsible for one critic. Therefore, CRA is critical to TDDR, and removing CRA would make TDDR ineffective.

\section{Conclusion}
\label{sec:Conclusion}

In this paper, we combine double actors with double critics to achieve better Q-value estimation and propose the DA-CDQ and CRA based on TD error, leading to the introduction of TDDR. TDDR significantly outperforms both TD3 and DDPG. Notably, when one actor is removed, TD3 and TDDR become equivalent. Another advantage of TDDR is that it does not require additional hyperparameters beyond those in TD3, unlike the recently proposed value estimation methods DARC, SD3, and GD3. These three algorithms introduce varying numbers of additional hyperparameters relative to TD3. The comparative performance indicates that careful tuning of hyperparameters is crucial for satisfactory learning performance in these algorithms, a requirement that can be avoided with TDDR. Future work will explore the feasibility of applying TDDR concepts to other reinforcement learning frameworks.

\appendix

The law of total expectation and the law of total variance \cite{Saltelli2008} are referenced here, as they are utilized in the proof of Theorem~\ref{theorem123}. If a random variable $A$ can take the values of other random variables $B$ and $C$, specifically, if $A$ equals $B$ with probability $p$ and $C$ with probability $1-p$, then $\bE[A]$ can be calculated using the law of total expectation:
\begin{align} 
\bE[A] = p\bE[B] + (1-p)\bE[C]. \label{EABC}
\end{align}
Similarly, the variance $\Var(A)$ can be determined using the law of total variance:
\begin{align}
    \Var(A) = &p\Var(B) + (1-p)\Var(C) \nonumber \\
    & + p(1-p)(\bE[B] - \bE[C])^2. \label{varABC}
\end{align}

\bibliographystyle{IEEEtran}
\bibliography{IEEEabrv,mychhrefs}

\end{document}